\documentclass[11pt]{article}

\usepackage[margin=1in]{geometry}

\usepackage[numbers, sort&compress]{natbib}
\usepackage{graphicx} %
\usepackage{amsfonts,amsmath,amssymb,amsthm}
\usepackage{bm}
\usepackage{caption}
\usepackage{subcaption}
\usepackage{wrapfig}
\usepackage{tabularx}
\usepackage[dvipsnames]{xcolor}
\usepackage[colorlinks = true, allcolors=NavyBlue]{hyperref}
\usepackage[capitalize,noabbrev]{cleveref}
\usepackage{algorithm}
\usepackage{algpseudocode}
\usepackage{pythonhighlight}
\usepackage{float}
\usepackage{enumitem}

\usepackage[T1]{fontenc}
\usepackage{newpxtext,newpxmath}

\usepackage{subcaption}
\usepackage{caption}
\usepackage{authblk}

\newtheorem{lemma}{Lemma}
\newtheorem{remark}{Remark}

\floatstyle{plain}
\newfloat{codeblk}{thp}{lop}
\floatname{codeblk}{Code Block}

\DeclareMathOperator*{\argmax}{arg\,max}

\begin{document}

\title{Spectral regularization for \\ adversarially-robust representation learning}

\author[1]{Sheng Yang\thanks{\href{mailto:shengyang@g.harvard.edu}{shengyang@g.harvard.edu}}}
\author[1,2,3]{Jacob A. Zavatone-Veth\thanks{\href{mailto:jzavatoneveth@g.harvard.edu}{jzavatoneveth@g.harvard.edu}}}
\author[1,2,4]{Cengiz Pehlevan\thanks{\href{mailto:cpehlevan@seas.harvard.edu}{cpehlevan@seas.harvard.edu}}}

\affil[1]{John A. Paulson School of Engineering and Applied Sciences,  Harvard University, Cambridge, MA}
\affil[2]{Center for Brain Science,  Harvard University, Cambridge, MA}
\affil[3]{Department of Physics,  Harvard University, Cambridge, MA}
\affil[4]{Kempner Institute for the Study of Natural and Artificial Intelligence, Harvard University, Cambridge, MA}

\date{\today}

\maketitle

\begin{abstract}
    The vulnerability of neural network classifiers to adversarial attacks is a major obstacle to their deployment in safety-critical applications. Regularization of network parameters during training can be used to improve adversarial robustness and generalization performance. Usually, the network is regularized end-to-end, with parameters at all layers affected by regularization. However, in settings where learning representations is key, such as self-supervised learning (SSL), layers after the feature representation will be discarded when performing inference. For these models, regularizing up to the feature space is more suitable. To this end, we propose a new spectral regularizer for representation learning that encourages black-box adversarial robustness in downstream classification tasks. In supervised classification settings, we show empirically that this method is more effective in boosting test accuracy and robustness than previously-proposed methods that regularize all layers of the network. We then show that this method improves the adversarial robustness of classifiers using representations learned with self-supervised training or transferred from another classification task. In all, our work begins to unveil how representational structure affects adversarial robustness. 
\end{abstract}

\section{Introduction}

Neural networks are vulnerable to adversarial attacks \cite{biggio2013evasion,szegedy2013intriguing}. For classification tasks, an originally correctly-classified sample may be recognized incorrectly after adding a perturbation so small as to be imperceptible to the human eye \cite{goodfellow2014explaining,szegedy2013intriguing}. Even without access to the model parameters and only access to inputs and outputs, attackers can still fool the network \citep{liu2018adv,li2022review,kurakin2018adversarial}. Identifying effective training algorithms that guard against these black-box attacks has therefore garnered widespread attention in the last decade \citep{athalye2018obfuscated,madry2018towards,athalye2018obfuscated,gao2019convergence,xie2017mitigating,liu2018adv,cohen2019certified}. Developing training mechanisms that provably provide black-box adversarial robustness is thus crucial in the development of practically deployable machine learning. 

However, recent years have seen growing popularity of representation learning paradigms to which many previous adversarial defenses are not immediately applicable. For instance, in transfer learning and self-supervised learning, the final linear readout layer of a network is retrained when performing downstream inference tasks. Only layers up to the feature representations are kept from the pretraining stage \cite{zbontar2021barlow,chen2020simclr,radford2021clip,hernandez2021scaling}. As standard adversarial training methods for classification networks typically require access to the last layer linear heads \cite{hein2017formal,yoshida2017spectral}, they cannot be directly applied in representation learning settings.

In this paper, we seek an adversarial training methodology that can be applied to these representation-focused training paradigms, in which the representation is pre-trained without knowledge of downstream tasks. Our primary contributions are organized as follows: We begin by introducing relevant previous works in \Cref{sec:related}. Then, in \Cref{sec:method}, we derive a theoretical guarantee on black-box adversarial robustness based on feature representations. This bound inspires our proposed regularizer \Cref{eq:eig-ub}, which penalizes the top singular value of each hidden layer's weights (i.e. excluding the last layer's linear head). In \Cref{sec:experiments}, we empirically demonstrate the effectiveness of this regularizer. We first show that it enlarges adversarial distances even more than regularizers incorporating the last layer for end-to-end classification training. Then, we show that it improves both adversarial robustness and test accuracy in both self-supervised and transfer learning across a range of simulated and moderate-scale image classification tasks. Our results provide evidence that having robust representations is crucial to adversarial robustness.

\section{Related Works} \label{sec:related}

This section provides a brief overview of the adversarial robustness literature, with a focus on black-box defenses. Two major approaches to improve black-box adversarial robustness have been proposed: training with adversarial examples and regularization \cite{li2022review,bai2021recent}. To train with adversarial samples, in each update to the network parameters, the training set is augmented with adversarial examples. Since we have access to input gradients during model training, one can use white-box attacks to find these examples \cite{goodfellow2014explaining,kurakin2018adversarial,madry2018towards,athalye2018obfuscated,gao2019convergence,xie2017mitigating,liu2018adv,cohen2019certified}. By forcing the model to become robust to these perturbations during training, it becomes less susceptible to future adversarial attacks. 

However, training with adversarial examples is computationally expensive, and it does not guarantee that the final classifier will be adversarially robust because of its strong dependence on the training dataset \cite{hein2017formal}. A less data-dependent and more computation friendly method is to design regularizers that encourage robustness. By adding specially designed regularization terms, the model can escape bad, non-robust local minima during optimization \citep{liu2020bad}. For linear regression, logistic regression, and decision trees with known uncertainty set structure, an exact equivalence between robustness and regularization has been established \cite{bertsimas2019machine,bertsimas2022robust}. In more advanced applications, one can derive regularizers that promote raising lower bounds on the adversarial distance \cite{hein2017formal,bhagoji2019lower,dohmatob2020classifier}. Our analysis attempts to generalize this adversarial robustness notion further for newer classification architectures.

In either case, searching for an adversarial sample with minimal adversarial distance in a black-box fashion is highly nontrivial. This makes black-box robustness evaluation rather difficult in practice, meaning that it is hard to evaluate defenses conclusively. Many query-based heuristic methods have been proposed \cite{brendel2017decision,chen2020hopskipjumpattack,li2020qeba,liu2019geometry,yan2020policy,ma2021finding,reza2023cgba}, which rely on iterative search for the closest point on the decision boundary of a trained model to a given sample. In our analysis we employ one such method, the tangent attack (TA) \cite{ma2021finding}, to evaluate model adversarial robustness.

\section{Spectral Regularization for Adversarial Robustness}\label{sec:method}

In this section, we build our spectral regularization on representations step-by-step. We first introduce the concept of adversarial distance and adversarial robustness; then we derive a lower bound on adversarial distance based on feature representations inspired by previous work of \citet{hein2017formal}; lastly, we derive the proposed regularizer.

\subsection{Preliminaries}
We first introduce the notion of black-box adversarial robustness. Consider the problem of assorting $n$-dimensional data into $K$ given classes. For a classification network $F(x; \Theta): \mathbb{R}^n \times \mathbb{R}^{|\Theta|} \rightarrow \mathbb{R}^K$ with trainable parameters $\Theta$, the class prediction is given by $\hat{y} = \argmax_{k \in [K]} \; F_k(x; \Theta)$, which is correct if it agrees with the true class label $y$. In the following discussions we drop the $\Theta$ notation when there is no ambiguity. Here we assume the output logits are distinct. 

\vspace{1em}
\noindent
\textbf{Adversarial Perturbation.} Consider a perturbation in the input space $\delta_x \in \mathbb{R}^n$ to a correctly-classified sample $x$. We say $\delta_x$ is an adversarial perturbation to $x$ if it swaps the predicted class label, i.e., there is a class index $k \neq y$ such that $ F_k(x + \delta_x) > F_y(x + \delta_x)$. 

\vspace{1em}
\noindent
\textbf{Adversarial Distance.} The adversarial distance $\Delta_x$ is defined as the minimal size of an adversarial perturbation for the datum $x$: $\Delta_x = \min_{\{\delta_x \in \mathbb{R}^{n}: \argmax_{k \in [K]}\; F_k(x + \delta_x) \ne y\}} \Vert \delta_x\Vert_2$. In this paper we focus on $l_2$ norm but the analysis can be generalized to other norms as well.

\vspace{1em}
\noindent
\textbf{Adversarial Robustness.} A network is adversarially robust to input perturbation with respect to sample $x$ if $\Delta_x$ is large, and is globally adversarial robust if $\Delta_x$ is large for all correctly classified $x$. 

\vspace{1em}
Note that we do not consider perturbing an incorrectly predicted sample. 
With this objective, past works have observed a tradeoff between adversarial robustness and classification accuracy \citep{goodfellow2014explaining,szegedy2013intriguing,zhang2019theoretically,raghunathan2020understanding}. We demonstrate this tradeoff through a simple binary classification task in \Cref{fig:acc_adv_tradeoff}. The third scenario is the ideal case, where we reach a globally adversarially robust state. This example motivates us to search for a solution that has large $\Delta_x$ for all correctly predicted samples $x$. Given a desired accuracy level $\rho \in [0, 1]$, define the set of parameters $S_{\rho, D, F}$ for the network $F$ that classifies a dataset $D = \{(x_i, y_i)\}_{i=1}^m, (x_i, y_i) \in \mathbb{R}^n \times [K]$ with accuracy at least $\rho$. Then, we are interested in finding the point within $S_{\rho,D,F}$ that maximizes the minimum adversarial distance: 
\begin{equation}
    \Theta^{(\text{robust})} = \underset{\Theta \in S_{\rho, D, F}}{\argmax} \min_{i \in [m]} \Delta_{x_i} .
    \label{eq:maxmin}
\end{equation}

In practice, without access to the test set, we hope that having larger adversarial distances in the training set translates to larger distances in the test set, which we validate empirically. The rest of this section thus gives a lower bound on $\Delta_x$ using spectral analysis, and derives a regularizer that encourages a neural network to look for parameters that have large such lower bounds. 

\begin{figure}[t]
    \centering
    \includegraphics[width=0.9\linewidth]{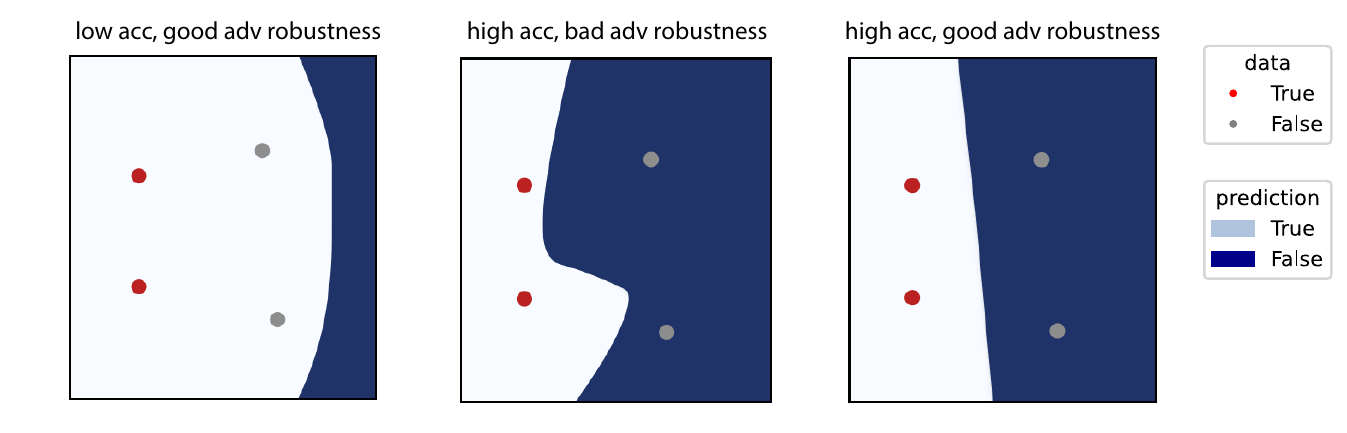}
    \setlength{\belowcaptionskip}{-1em}
    \caption{Tradeoff between prediction accuracy and adversarial robustness in a simple binary classification task. The first scenario has 50\% accuracy, but the adversarial distances for the correct samples are large. The second case correctly classifiers every sample, but does not have large adversarial distances for each. The last scenario is the ideal: a large-margin solution.}
    \label{fig:acc_adv_tradeoff}
\end{figure}

\subsection{A lower bound on adversarial distance}

Given a classification network $F(x): \mathbb{R}^n \rightarrow \mathbb{R}^K$ of $L$ layers, we can decompose it into a feature map $\Phi(x): \mathbb{R}^n \rightarrow \mathbb{R}^d$, where $d$ is the feature dimension, and a linear decision head $f(z): \mathbb{R}^d \rightarrow \mathbb{R}^K$: 
\begin{equation}
    F = \texttt{SoftMax} \circ f \circ \Phi
    \label{eq:decomp}
\end{equation}

With this decomposition, we present our central lemma, which is an extension of the main result of \citet{hein2017formal}. The objective of our analysis is to isolate the influence of the readout weights, as we focus on settings where the representation is learned in a pretraining phase and the readouts are trained separately to perform downstream tasks. 

\begin{lemma}\label{lemma:bound}
    Suppose a neural network classifier $F: \mathbb{R}^n \rightarrow \mathbb{R}^K$ is continuously differentiable and can be decomposed as \Cref{eq:decomp}, with $f(z) = W^{(L)}z$ the last layer linear transformation without bias, and $W^{(L)}_c$ the $c$-th row of $W^{(L)}$, treated as a row vector. Suppose $x \in \mathbb{R}^n$ is a sample input belonging to class $c$ and $\delta_x \in \mathbb{R}^n$ an adversarial perturbation to $x$ that results in an incorrect prediction of class $k$. Fixing a perturbation radius `budget' $R>0$ such that $\Vert\delta_x \Vert_2 \leqslant R$, we have 
    \begin{equation}
         \Delta_x  \geqslant \theta_x \Vert \Phi(x) \Vert_2 \cdot \frac{1}{\max_{y \in B_2(x, R)} \Vert \nabla \Phi(y)\Vert_2}\label{eq:bound_rewrite}
    \end{equation}
    where $\theta_{x} = \frac{(W^{(L)}_c - W^{(L)}_k ) \Phi(x)}{\Vert W^{(L)}_c - W^{(L)}_k\Vert_2\Vert \Phi(x)\Vert_2}$ is the cosine of the angle between the feature representation of $x$ and the last layer "confidence weights" $W^{(L)}_c - W^{(L)}_k $ for class $c$ relative to class $k$. Here, $\nabla\Phi(\cdot)$ is the Jacobian of the feature map with respect to the input. 
\end{lemma}
\begin{proof} See \Cref{proof:lemma1} for the proof of this lemma, extended to general $l_p$ norms. 
\end{proof}

In this lower bound, the key term of interest from a feature representation perspective is the norm of the feature map's gradient $\Vert\nabla \Phi(\cdot)\Vert_2$ in the denominator. As it depends on the readout weights, the angle $\theta_x$ cannot be controlled during the pretraining phase when features are learned. When a readout is subsequently trained to perform a downstream task, $\theta_x$ is determined by the algorithm chosen at that stage. We discuss this point, and comment on the dynamics of $\theta_x$ in the supervised setting where the readout and representation are learned simultaneously, in \Cref{sup:last-layer}. We also remark that we make the choice to regularize the gradient norm directly rather than explicitly also attempting to control the scale of the representation $\Vert \Phi(\cdot) \Vert^2$. This choice is motivated by the fact that a global re-scaling $\Phi(x) \mapsto c\ \Phi(x)$ could be compensated by downstream readout, while a spatially-anisotropic rescaling will also change the gradient norm. 

This lower bound suggests that we should penalize $\max_{y \in B_2(x, R)} \Vert\nabla \Phi(y)\Vert_2$ during representation learning. This contrasts with previous works on supervised learning that penalize the input-output Jacobian $\nabla F$ \cite{hein2017formal,yoshida2017spectral,hoffman2019robust}. However, to make regularization computationally efficient, we must make several relaxations of this objective. First, we consider the norm of the gradient at $x$ only, since searching over the whole of the ball is computationally expensive and not easily parallelizable. Noting that the 2-norm of a matrix is its maximum singular value, we have
\begin{equation}
    \Vert \nabla \Phi(x)\Vert_2^2 = \lambda_{\max} (\nabla \Phi(x) \nabla \Phi(x)^T) = \lambda_{\max} (\nabla \Phi(x)^T \nabla \Phi(x)) =: \lambda_{\text{max}}(g) ,\label{eq:lambda_max_g}
\end{equation}
where we have defined $g:= \nabla \Phi(x)^T \nabla \Phi(x)$. Here, $\lambda_{\text{max}}(\cdot)$ denotes the maximum eigenvalue of a matrix. A robust classification network thus must have small $\lambda_{\max}(g)$. 

\begin{remark} \label{rmk:metric}
    \textbf{(Metric Tensor)} If we consider $\Phi$ as the feature map and pull the Euclidean metric from feature space back to input space, the induced metric tensor is exactly $g$ as defined above. The behavior of the pull-back metrics induced by trained deep network feature maps has been studied by \citet{zavatone2023neural}, focusing on enlargement of the volume element $\sqrt{\det g}$ near decision boundaries. The conclusion in our work that increasing $\lambda_{\max} (g)$, the dominant term in $\sqrt{\det g}$, encourages adversarial robustness, adds theoretical context to the empirical observations of \cite{zavatone2023neural}. 
\end{remark}

In principle, $\lambda_{\max}(g)$ is differentiable so long as the spectral gap is nonzero, and so can be na\"ively added to the loss term as a regularizer. However, computing the resulting updates using automatic differentiation suffers from high runtime and memory consumption \citep{zavatone2023neural}. To alleviate such costs, we perform a series of relaxations. Assuming that the activation function has derivatives (almost everywhere) bounded by 1, we can bound $\lambda_\text{max}(g)$ by the product of the largest singular values squared of the hidden layer weights $W^{(l)}$, $\sigma_{\text{max}}^2(W^{(l)})$, where $l \in [L-1]$ for $L$ the the total number of layers. This trick was used before by \citet{yoshida2017spectral}. The resulting regularizer, which we henceforth refer to as \texttt{rep-spectral} because it is a \texttt{spectral} regularizer on \texttt{rep}resentations, is
\begin{equation} \textstyle
    L^{(\texttt{rep-spectral})}(\Theta) = \frac{\gamma}{2} \sum_{l=1}^{L - 1} \sigma_{\max}^2(W^{(l)}), \label{eq:eig-ub} . 
\end{equation}
where $\gamma \geqslant 0$ is the regularization strength. We present the full derivation of this regularizer in \Cref{sup:feat-reg}. One may use log on the singular values as directly regularizing the product of singular values, but this would incur additional computational costs and result in a data-dependent scaling factors in gradient computation that could be heuristically absorbed in $\gamma$. The difference with \citep{yoshida2017spectral} is that we drop the penalization of the last layer. We refer to their proposed regularizer as \texttt{ll-spectral} to emphasize that it penalizes the top singular value of the readout weights, in contrast to our method. Concretely, their regularizer is
\begin{equation} \textstyle
    L^{(\texttt{ll-spectral})}(\Theta) = \frac{\gamma}{2} \sum_{l=1}^{L} \sigma_{\max}^2(W^{(l)}); 
\end{equation}
we will compare the effects of these two regularizers on supervised pretraining in \Cref{sec:experiments}. 

To efficiently compute the parameter update based on this regularization, note that the gradient of the top singular value can be computed analytically under the assumption that it is unique as ${\partial \sigma^2_{\max}(W)}/{\partial W} = 2 \sigma_{\max} u_{\max} v_{\max}^T$, where $u_{\max}, v_{\max}$ are the left and right singular vectors for $\sigma_{\max}$. One can then employ power iterations with singular direction updates amortized across parameter updates \citep{yoshida2017spectral} instead of automatic differentiation to efficiently adjust the parameters (e.g. perform 1 power iteration every 10 parameter updates when learning rate is small, see \Cref{sup:power-iteration}).

\begin{remark} \textbf{(Extension to Convolutional Neural Networks)}
The assumption that $F$ consists of only fully-connected layers can be relaxed. For deep convolutional neural networks, we can compute the maximum singular values of matrix representations of the convolutional layers. Concretely, suppose we are given an input image $X \in \mathbb{R}^{c_{\text{in}} \times n \times n}$ and a multi-channel filter $\mathcal{K} \in \mathbb{R}^{c_{\text{out}} \times c_{\text{in}} \times k \times k}$. Then, there exists a matrix $\mathcal{\tilde{K}} \in \mathbb{R}^{c_{\text{out}}n_{\text{out}}^2 \times c_\text{in}n^2}$ such that a periodic 2D convolution can be linearized as $\operatorname{Vec}(\texttt{Conv2D}(X)) = \mathcal{\tilde{K}} \; \operatorname{Vec}(X)$, where $\operatorname{Vec}(.)$ is a row-major flattening of the input image $X$. We can then extract the top eigenvalue of $\tilde{K}$ for regularization. See details in \Cref{sup:linearize_cnn} and in \cite{sedghi2018singular,senderovich2022towards} for the construction of $\tilde{K}$ and algorithms to extract its top singular value both exactly and iteratively.
\end{remark}

\section{Experiments}\label{sec:experiments}

We now evaluate our proposed \texttt{rep-spectral} regularizer based on the accuracy of classification and the adversarial distance of selected samples at the end of training. For benchmarking purposes and to gain intuition for its effect on pretraining, we first compare it against the \texttt{ll-spectral} regularizer in fully supervised settings where there is no separate training. We then turn to the pretraining settings which motivate our work. In all cases, we regard a regularizer to be successful if we observe an increase in average adversarial distances without a significant drop in test accuracy. We will see that there are some cases where the test accuracy in fact improves with regularization.  

To evaluate black-box adversarial robustness, we use the Tangent Attack (TA) method \cite{ma2021finding}. As introduced in \Cref{sec:related}, the TA method seeks an approximation to $\delta_x$, an adversarial perturbation to a correctly classified sample $x$ that is as small as possible in $l_2$ norm. The TA algorithm uses 2D analytic geometry to iteratively locate the point nearest to $x$ that is assigned a different label by the network (i.e., we focus on non-targeted attacks). We provide a more detailed discussion of the TA method in \Cref{sup:ta}, and document training and evaluation details in \Cref{sup:exp_details}. Code to reproduce all experiments is available on GitHub.\footnote{\url{https://github.com/Pehlevan-Group/rep-spectral}}

\subsection{Shallow MLPs trained on a toy dataset}

\begin{wrapfigure}{R}{0.5\linewidth}
    \centering
    \includegraphics[width=\linewidth]{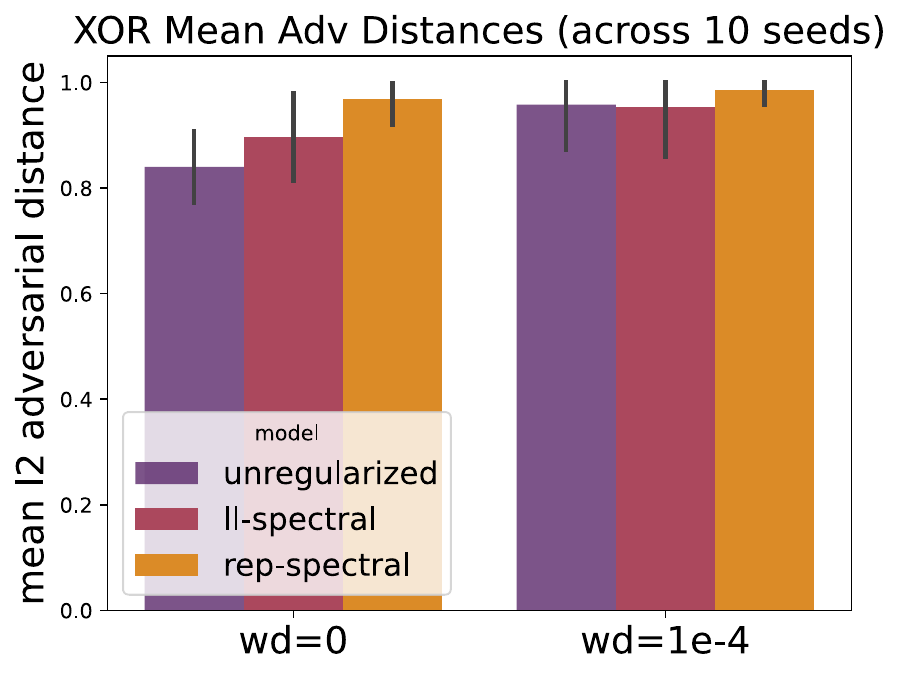}
    \caption{Average $\Delta_x$ found by TA across 10 different seeds for models trained on XOR with (\emph{right}) and without (\emph{left}) weight decay, and with the inclusion of our proposed \texttt{rep-spectral} regularizer, the \texttt{ll-spectral} regularizer that includes all layers, or no additional spectral regularizer. Error bars show $\pm 1$ standard deviation over seeds. }
    \label{fig:xor_clean_seed}
\end{wrapfigure}

To gain intuition for how our regularizer shapes representations, we first apply it to a single-hidden-layer MLP trained on a toy 2D XOR task. Though this task is of course unrealistic, it is potentially useful because we can directly visualize the input space. Given 4 data points at $[\pm 1, \pm 1] \in \mathbb{R}^2$, we use a network with 8 hidden units and GELU nonlinearity. Though our regularizer is motivated by the independent-pretraining, we find that even in this fully-supervised setting we obtain improved adversarial robustness, even compared to training using the \texttt{ll-spectral} regularizer that penalizes the last layer (\Cref{fig:xor_clean_seed}). More training details and results can be found in \Cref{sup:shallow}.

How does this robustness arise? As the input space for the XOR task is two-dimensional, we can directly visualize it. Examining the decision boundaries in \Cref{fig:xor_clean}, we see that our \texttt{rep-spectral} regularizer results in increased classification margin, while the \texttt{ll-spectral} regularizer does not substantially affect the decision boundary. This difference reflects the fact that \texttt{ll-spectral} mostly just penalizes the readout layer norm and fails to control the feature layer norm (\Cref{fig:weight_norm}). To gain a more detailed understanding of how the representations differ, we visualize the volume element corresponding to the metric induced by the feature map, as described in \Cref{rmk:metric}. In \Cref{fig:xor_volume}, we see that the \texttt{rep-spectral} regularizer noticeably increases the areas of small volume element (and thus low representational sensitivity) near the class centers relative to the unregularized and \texttt{ll-spectral} models.

\begin{figure}[t]
    \centering
    \begin{subfigure}{1.0\textwidth}
        \centering
        \caption{}
        \label{fig:xor_clean}
        \includegraphics[width=0.85\linewidth]{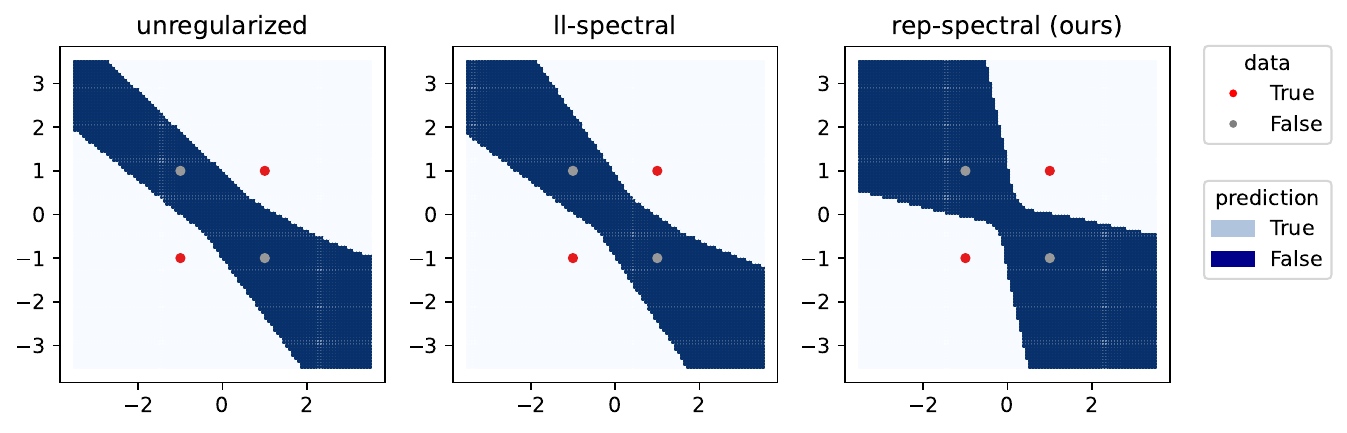}
    \end{subfigure}

     \begin{subfigure}{1.0\textwidth}
        \centering
        \caption{}
        \label{fig:xor_volume}
        \includegraphics[width=1\linewidth]{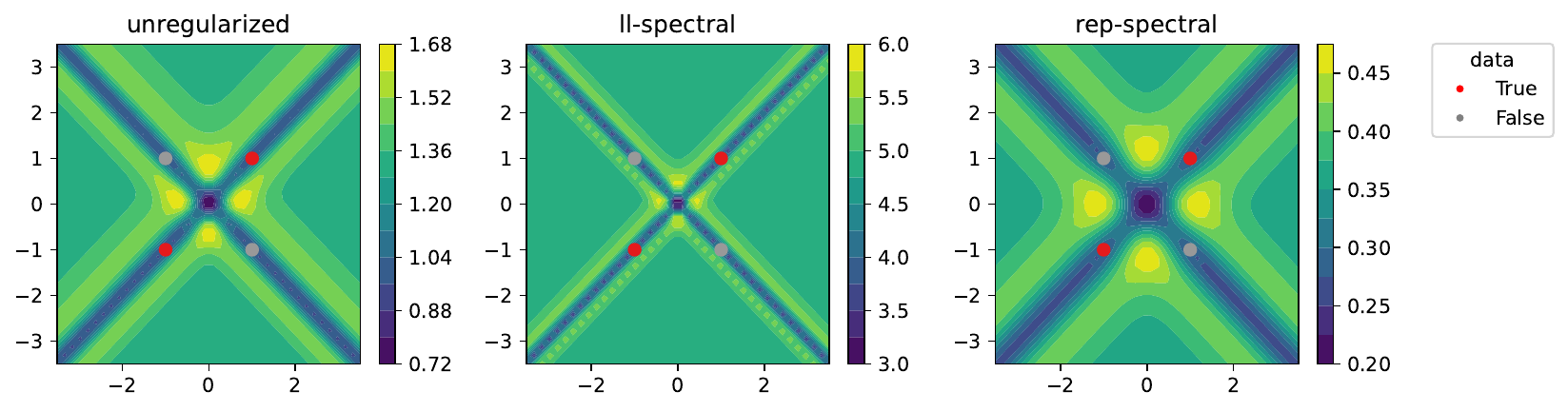}
    \end{subfigure}
    
    \caption{Effect of spectral regularization on the representations of MLPs trained on the toy XOR task. (a). Direct visualization of the decision boundaries of models trained using no regularization (\emph{left}), \texttt{ll-spectral} regularization (\emph{middle}), and our proposed \texttt{rep-spectral} regularization (\emph{right}). The four training points are shown, colored according to their class. (b). Visualization of the volume element (see \Cref{rmk:metric}), which measures the sensitivity of the representation to small variations in the input, for models trained with each of these three methods. For details, see \Cref{sup:shallow}.}
\end{figure}

\subsection{Shallow MLPs trained on MNIST images}

We next apply our regularizer to single-hidden-layer MLPs trained to classify MNIST images \cite{lecun2010mnist}, with flattened input of 784 dimensions and 2000 hidden units. We sample 1000 testing images and apply the TA algorithm to detect the minimum $l_2$ perturbation. In the fully supervised setting, though regularization slightly decreases test accuracy, we observe that excluding the last layer produces a smaller loss of accuracy and a larger increase in adversarial distance compared to regularizing all layers (\Cref{fig:mnist_adv}). We also visualize selected test images and their adversarial perturbations in \Cref{fig:mnist-perturbation} and observe an overall increase in scale of perturbation to make adversarial predictions compared to unregularized and spectral method. More training details and results are in \Cref{sup:shallow}.

To test whether our regularizer leads to more robust representations, we discard the linear head and perform classification using multilogistic regression on the fixed feature representations (see \Cref{sup:shallow} for details). Surprisingly, in this setting our regularizer does not hurt test accuracy compared to the unregularized model (\Cref{fig:mnist_adv_newhead}). Moreover, the resulting model is more adversarially robust than when supervised pretraining of the representation is performed without regularization or with \texttt{ll-spectral} regularization (\Cref{fig:mnist_adv_newhead}). This suggests that, on a simple image classification task, our method achieves its stated goal: to enable representation learning that gives good adversarial robustness when an unregularized readout is trained to perform downstream tasks.

\begin{figure}[t]
    \centering
    \begin{subfigure}{0.48\textwidth}
        \centering
        \includegraphics[height=180pt]{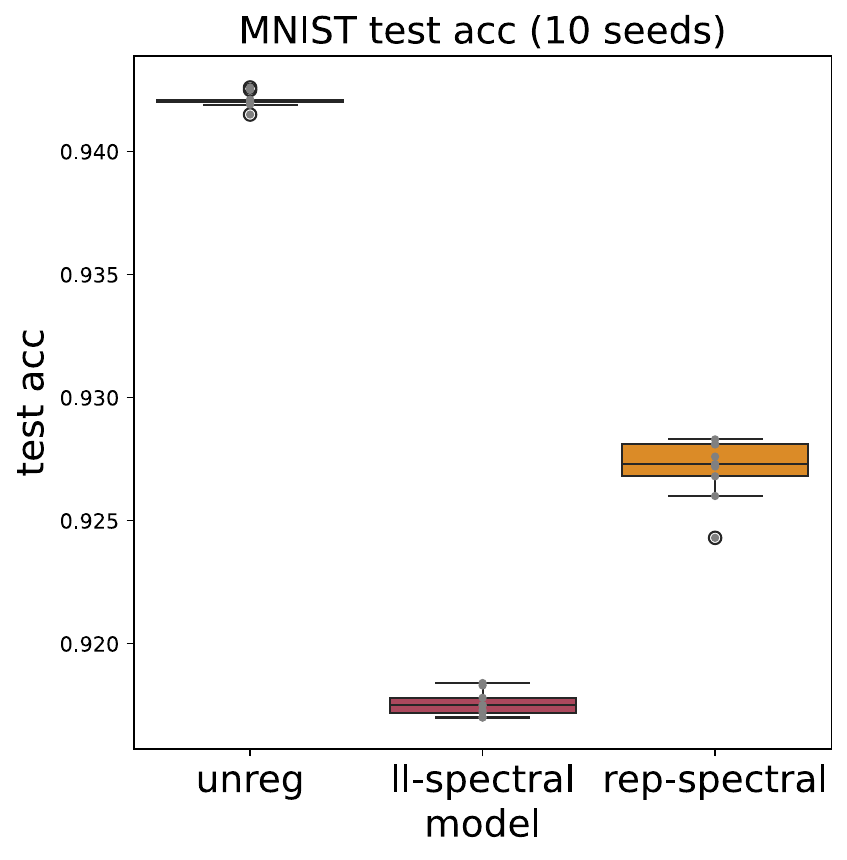}
    \end{subfigure} \hfill 
    \begin{subfigure}{0.48\textwidth}
        \centering
        \includegraphics[height=180pt]{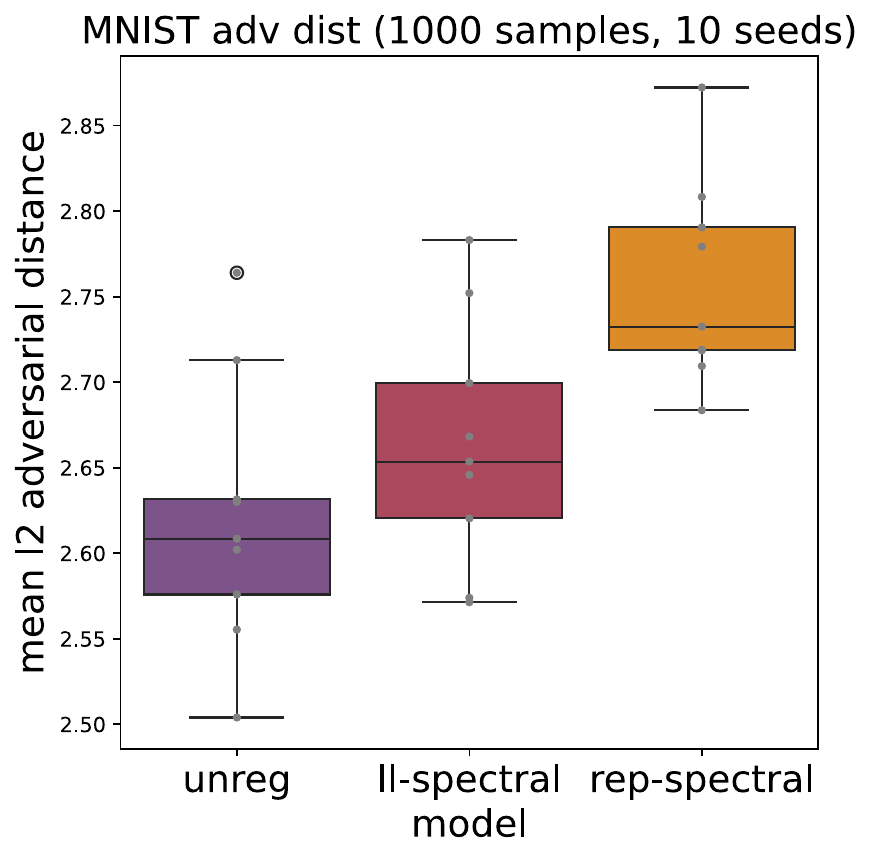}
    \end{subfigure}
    \caption{Spectral regularization during supervised training of a single-hidden-layer MLP on MNIST images improves robustness. For each regularization method, we show text accuracy (\emph{left}) and adversarial distance averaged across 1000 samples (\emph{right}) across 10 random seeds. Gray dots indicate the results for individual seeds, while boxplots show the mean and quartiles. }
    \label{fig:mnist_adv}
\end{figure}

\begin{figure}[t]
    \centering
    \begin{subfigure}{0.47\textwidth}
        \centering
        \includegraphics[height=180pt]{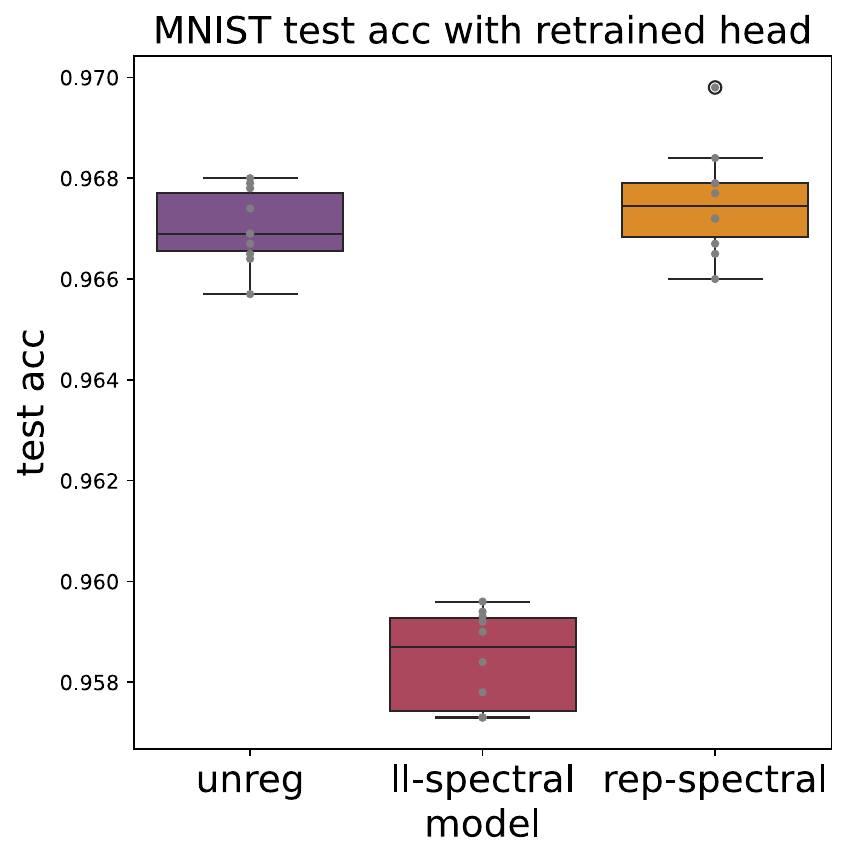}
    \end{subfigure} \hfill
    \begin{subfigure}{0.52\textwidth}
        \centering
        \includegraphics[height=180pt]{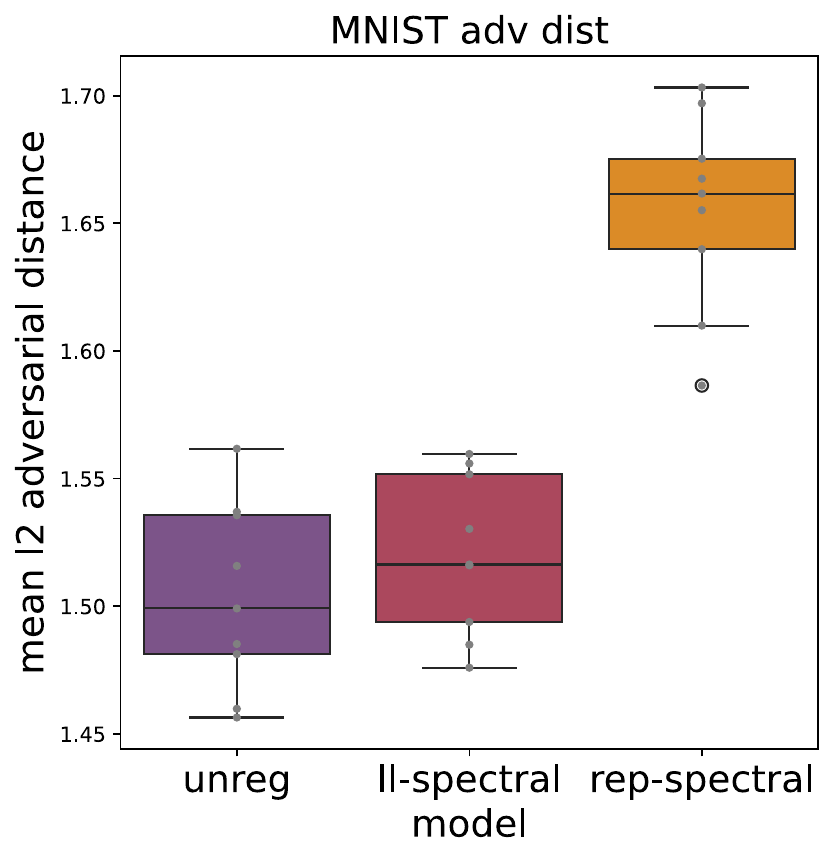}
    \end{subfigure}
    \caption{Re-training a readout from the hidden representation of an MLP pretrained on MNIST. For each regularization method, we show text accuracy (\emph{left}) and adversarial distance averaged across 1000 samples (\emph{right}) across 10 random seeds. Gray dots indicate the results for individual seeds, while boxplots show the mean and quartiles.}    
    \label{fig:mnist_adv_newhead}
\end{figure}

\subsection{Deep Convolutional Networks: ResNet}

As a final example in the supervised setting, we consider a deep convolutional architecture. We train ResNet18s \citep{he2016deep} to classify CIFAR-10 images \cite{krizhevsky2009cifar} using different adversarial regularization methods (see \Cref{sup:deep_conv} for training details). We plot the distribution of adversarial distances of 1000 randomly selected test samples in \Cref{fig:resnet18_kde} and test accuracy and average l2 adversarial distance in \Cref{fig:resnet18_cifar10}. As in \citet{yoshida2017spectral}'s original proposal of \texttt{ll-spectral} regularization that incorporates the last layer, regularization provides a noticeable boost in test accuracy and robustness. Our \texttt{rep-spectral} method performs quite comparably to \texttt{ll-spectral}, with similar gains in test accuracy and perhaps marginally larger improvement in robustness. Interestingly, we see from the distributions of adversarial distance in \Cref{fig:resnet18_kde} that the gain in robustness over the unregularized model is mostly by enlarging the distances of samples that are originally close to the decision boundaries, as reflected by the suppression of the bump in the left tail of the density. See \Cref{fig:resnet18_dist_prop} for quantifying the distribution shift across different random seeds. This phenomenon is qualitatively consistent with our observations from the toy XOR task in \Cref{fig:xor_volume}. 

\begin{figure}[t]
    \centering
    \begin{subfigure}{0.42\textwidth}
        \centering
        \caption{}
        \label{fig:resnet18_kde}
        \includegraphics[width=\linewidth]{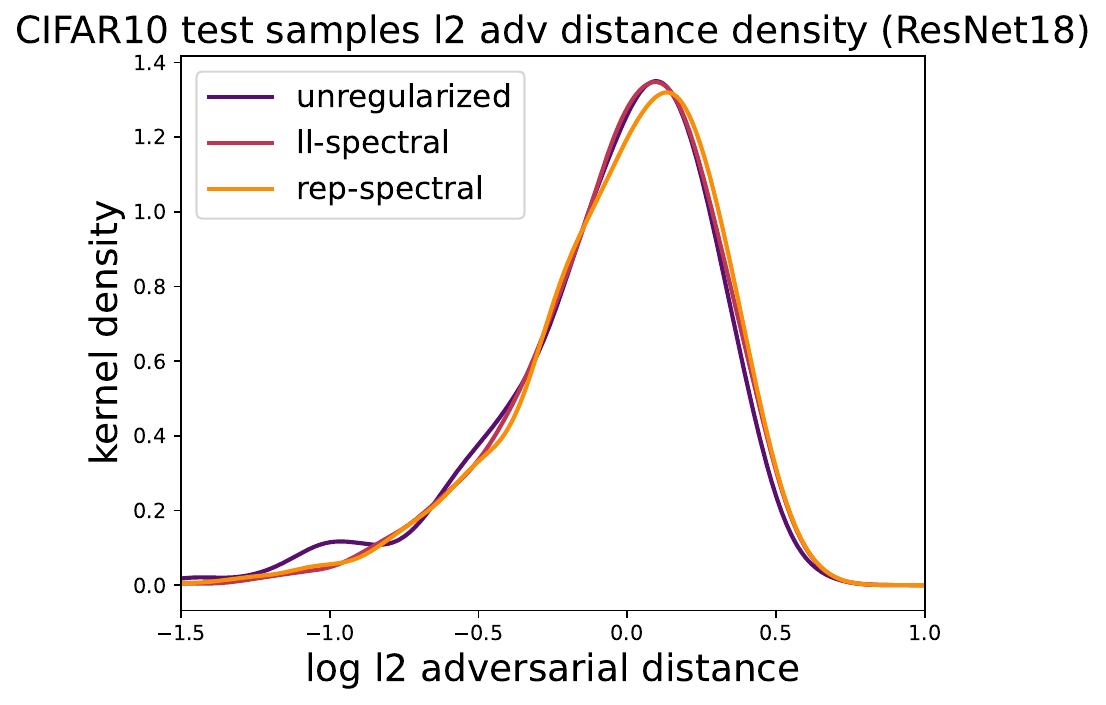}
    \end{subfigure} \begin{subfigure}{0.55\textwidth}
        \centering
        \caption{}
        \label{fig:resnet18_cifar10}
        \includegraphics[width=\linewidth]{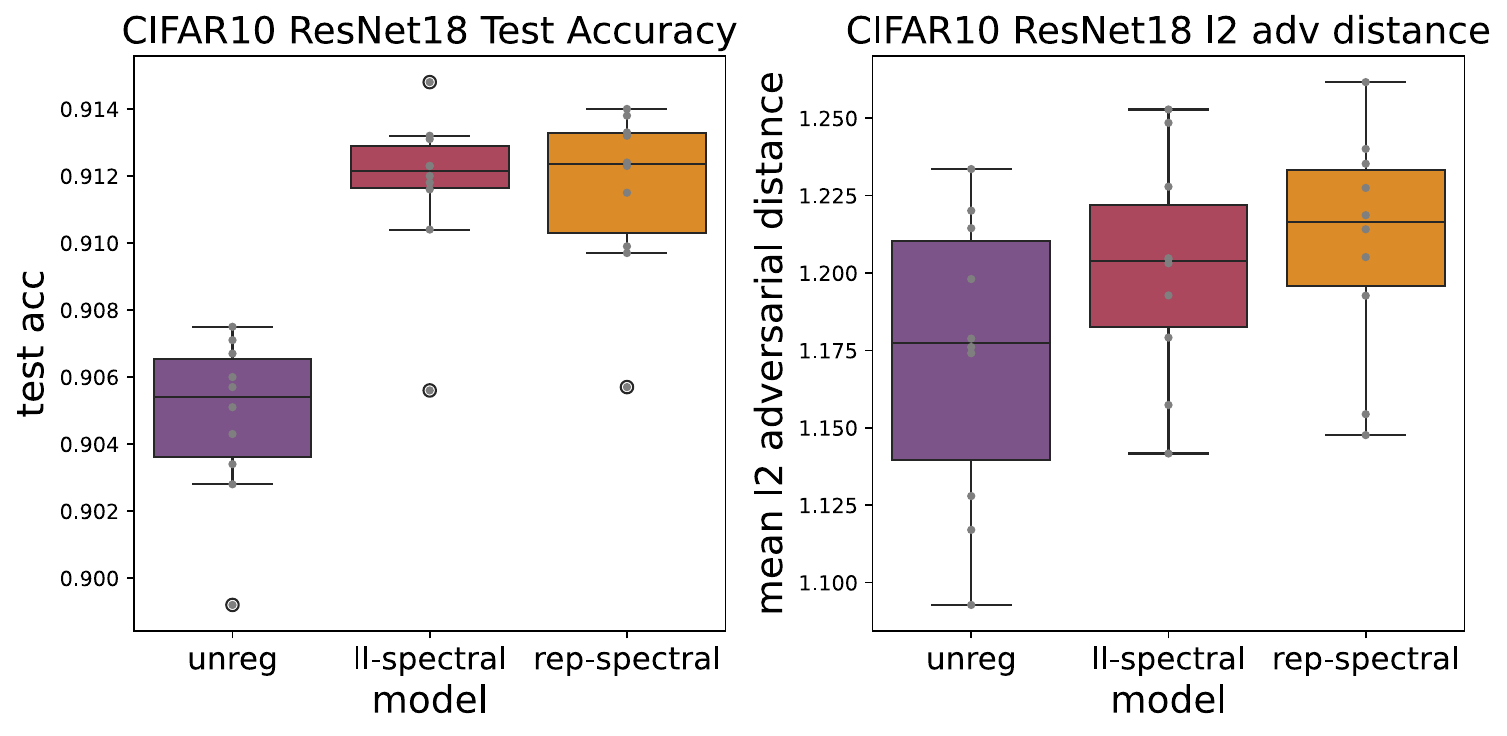}
    \end{subfigure}
    \caption{Spectral regularization during supervised training of ResNet18s to classify CIFAR-10 images improves accuracy and robustness. (a). Kernel density estimates of the distributions of adversarial distances across 1000 random samples for unregularized, \texttt{ll-spectral}, and \texttt{rep-spectral} models. Note that the distribution from \texttt{rep-spectral} shifts to the right of the distributions from unregularized and \texttt{ll-spectral} regularized models. (b). Test accuracy (\emph{left}) and mean adversarial distance (\emph{right}) for these models across 10 random seeds. Gray dots indicate the results for individual seeds, while boxplots show the mean and quartiles.}
\end{figure}

\subsection{Self-Supervised Learning: Barlow-Twins}

We now turn at last to self-supervised pretraining, which provided the motivation for our work. As noted before, self-supervised learning (SSL) focuses on obtaining meaningful representations for various downstream applications without explicit label information. As all components of the network after the feature layer will be discarded in these downstream tasks, we naturally only regularize up to the feature layer. Here, we consider using the popular contrastive SSL method BarlowTwins \cite{zbontar2021barlow} to learn representations of CIFAR-10, with a ResNet18 backbone. To evaluate accuracy and adversarial robustness on a downstream classification task, we train a multi-class logistic regression readout from the feature space. Full training details are provided in \Cref{sup:ssl}. As we saw in supervised settings, pretraining with our \texttt{rep-spectral} regularizer yields a slight improvement in test accuracy and adversarial robustness (\Cref{fig:ssl}). Though the gain in average adversarial robustness is not large (\Cref{fig:cifar10_barlow_resnet18}), by examining the distribution of adversarial distances we see that the robustness of the most vulnerable examples is improved (\Cref{fig:barlow_kde}). \looseness=-1

\begin{figure}[t]
    \centering
    \begin{subfigure}{0.41\textwidth}
        \centering
        \caption{}
        \label{fig:barlow_kde}
        \includegraphics[width=\linewidth]{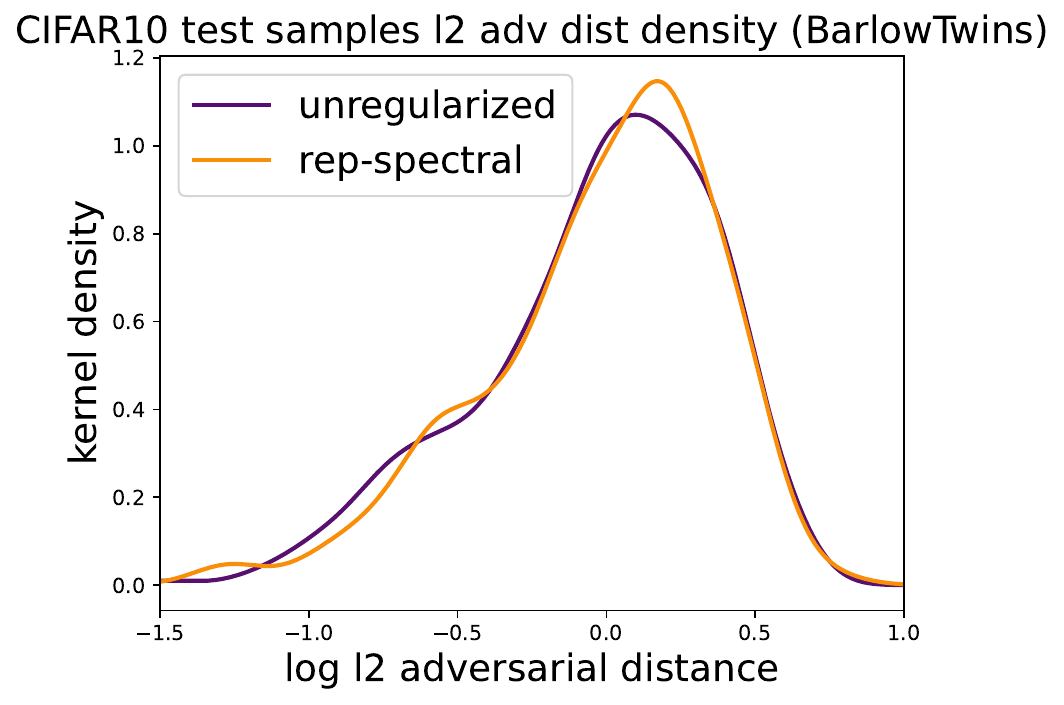}
    \end{subfigure} \begin{subfigure}{0.57\textwidth}
        \centering
        \caption{}
        \label{fig:cifar10_barlow_resnet18}
        \includegraphics[width=\linewidth]{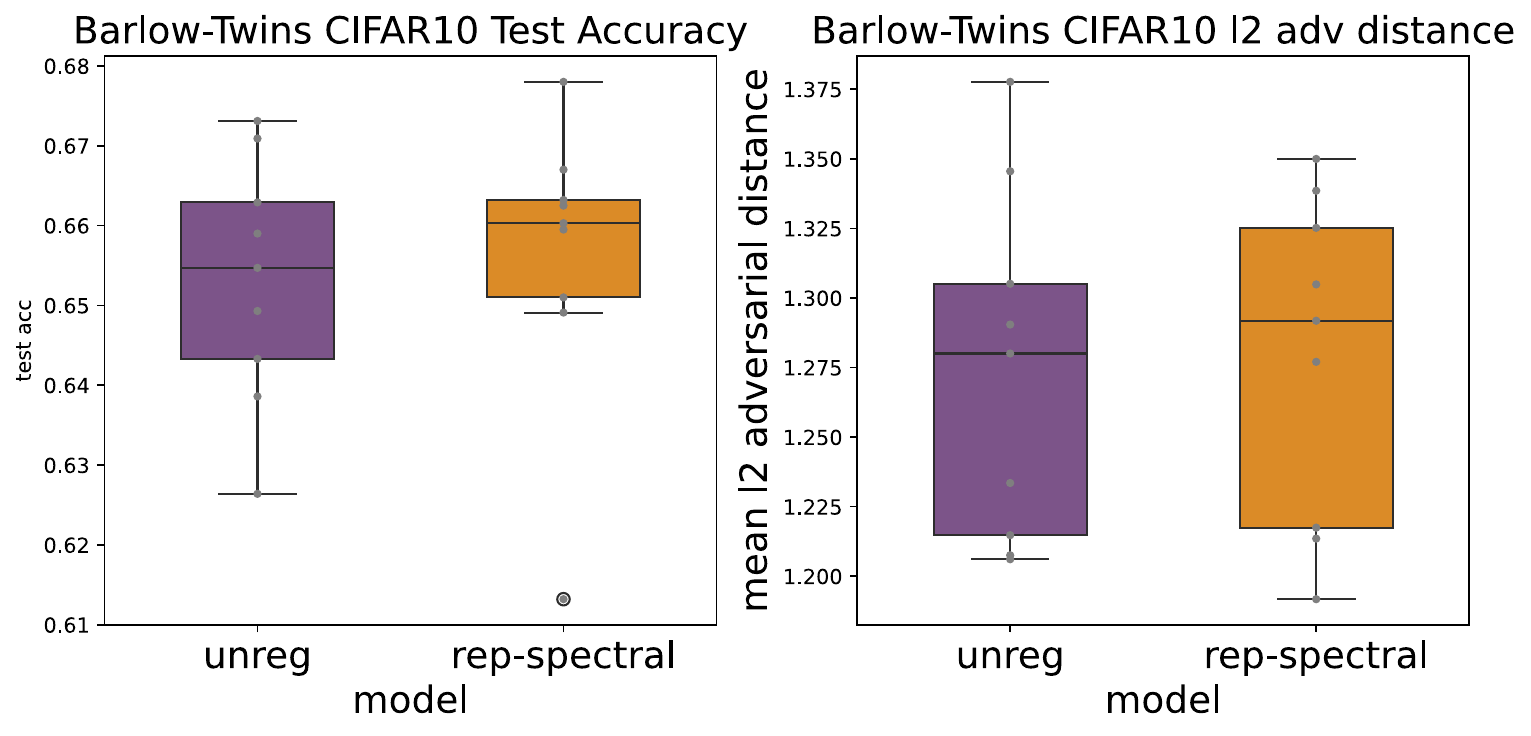}
    \end{subfigure} \\
    \caption{Spectral regularization during self-supervised BarlowTwins pretraining of a ResNet18 improves robustness of a readout in downstream classification. (a). Kernel density estimates of the distributions of adversarial distances across 1000 random samples for models pretrained without regularization and with \texttt{rep-spectral} regularization. (b). Test accuracy (\emph{left}) and mean adversarial distance (\emph{right}) for these models across 10 random seeds. Gray dots indicate the results for individual seeds, while boxplots show the mean and quartiles.}
    \label{fig:ssl}
\end{figure}

\subsection{Transfer learning}

Lastly, we apply our regularizer to transfer learning (TL) tasks. We pretrain ResNet50 on ImageNet-1K \citep{deng2009imagenet} and finetune on CIFAR-10 (see \Cref{sup:tl} for details). In the finetuning stage, the readout layer is trained from scratch, while only miniscule adjustments are made to the hidden layers. For comparison in this setting, we add an additional candidate regularizer used at the finetuning stage, batch spectral shrinkage (\texttt{BSS}) \citep{chen2019catastrophic}. We visualize the mean adversarial distances for 500 randomly test samples and report the mean over 10 random seeds in \Cref{fig:tl_supp} and corresponding test accuracy in \Cref{fig:tl_supp_acc}. Although all model reaches 96\% test accuracy consistently, they have dramatically different robustness level. We found adding \texttt{rep-spectral} regularization at the pretraining stage produces substantial gains in adversarial robustness, while adding adversarial regularization in the finetuning stage typically hurts adversarial robustness. The best robustness is obtained by adding our proposed regularizer during pretraining and then fine-tuning without regularization (compare the dark purple bar with the dark orange bar in \Cref{fig:tl_supp}). A similar pattern holds in finetuning on other dataset such as Stanford Dog \citep{KhoslaYaoJayadevaprakashFeiFei_FGVC2011}, Oxford Flowers \citep{Nilsback08}, and MIT indoor \citep{quattoni2009recognizing}; see  \Cref{sup:tl} and \Cref{fig:tl-others} for more details. Therefore, our proposed method yields a substantial gain in robustness in a transfer learning setting.

\begin{figure}[t]
    \centering
    \includegraphics[width=\textwidth]{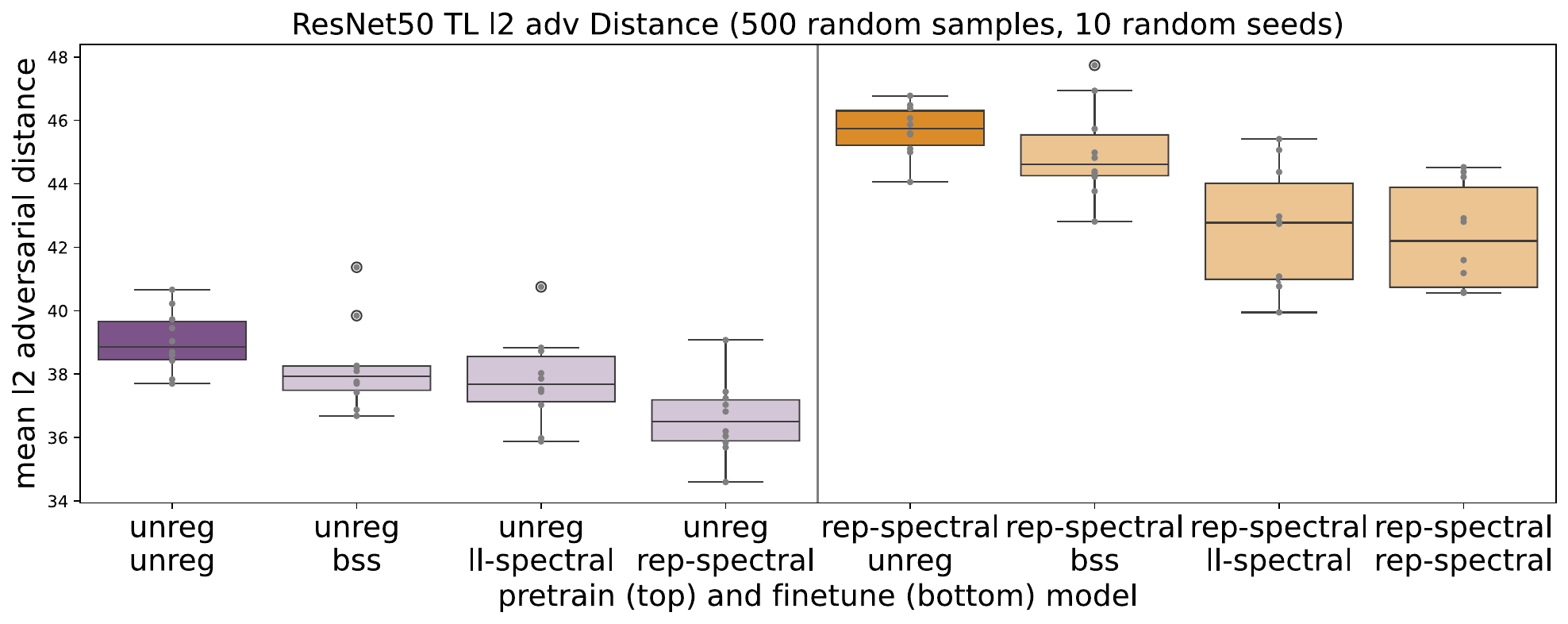}
    \caption{Mean $\Delta_x$ in transfer learning across different combinations of training schemes. The left half are finetuning from unregularized model, and the right half are finetuning from \texttt{rep-spectral} regularized model. In finetuning stage, adding additional regularization typically hurts adversarial robustness performance.}
    \label{fig:tl_supp}
\end{figure}

\section{Discussion}

In this paper, we have shown that a simple regularization method encourages adversarial robustness during representation learning. We proved that the robustness of a representation is linked to the spectral norm of the input-wise Jacobian of the feature map (\Cref{lemma:bound}). Inspired by this bound, we proposed a regularizer that penalizes the top singular value of neural network layers only up to the feature space. Through ample experiments, we have shown that our regularizer is more effective than a previously-proposed method incorporating the readout layer in encouraging $l_2$ adversarial robustness during end-to-end training. This gain in robustness mostly results from increasing the adversarial distances to samples that are previously close to the decision boundaries (\Cref{fig:resnet18_kde}). The pattern holds in supervised, semi-supervised and transfer learning, including as training on synthetic data and MNIST using a shallow architecture, on CIFAR10 using ResNet18 supervised training and BarlowTwins self-supervised training, and pretraining on ImageNet with subsequent finetuning on CIFAR10. In some experiments, we found that this new spectral regularizer even yields a gain in test accuracy.  

Our proposed method is not without limitations. We observed limited improvement in robustness in self-supervised learning (\Cref{fig:barlow_kde}), and saw that regularizing during the fintetuning stage of transfer learning can decrease robustness (\Cref{fig:tl_supp}). As our primary objective is to improve robustness of pretrained representations, the former limitation is more important. We note that adversarial robustness in SSL is in general not well-understood. It will be interesting to investigate whether there are ways to modify our training setup in which regularization yields substantial improvement in robustness. In particular, it will be interesting to investigate SSL pretraining using larger datasets, given the substantial improvements with representation quality observed with dataset scale in non-regularized settings \cite{chen2020simclr,zbontar2021barlow,radford2021clip}. 

This study opens the door for analyzing adversarial robustness on a per-layer basis. In deep networks, it is widely believed that early and later layers assume different roles in the learning tasks. Early layers may be responsible for low-level feature extraction, while later layers adapt to high-level and task-specific features \cite{cammarata2020thread,feather2023metamer,zavatone2023neural}. As analyzed in work by \citet{dyballa2024separability} on the generalization performance of neural networks from a per-layer perspective, an interesting extension of our would would be to conduct ablation studies by turning on and off spectral regularizations for certain layers, not necessarily contiguous ones. This could allow one to identify the crucial layers contributing to model adversarial robustness.

It is also worth comparing our analysis with another line of adversarial robustness research based on adversarial training in the black-box attack literature \cite{szegedy2013intriguing,goodfellow2014explaining,tramer2017ensemble,bai2021recent}. Adversarial training searches for adversarial samples after each parameter update and insert such samples into the training set, through which the adversarial robustness of the model is empirically injected. Although our method is data-independent and more parallel-computation-friendly, it is worth comprehensively comparing the runtime and adversarial robustness for adversarial attack methods and regularization methods that spontaneously provides adversarial robustness.

\section*{Acknowledgements}

We thank Alexander Atanasov and Benjamin S. Ruben for helpful comments. JAZV and CP were supported by NSF Award DMS-2134157 and NSF CAREER Award IIS-2239780. CP is further supported by a Sloan Research Fellowship. This work has been made possible in part by a gift from the Chan Zuckerberg Initiative Foundation to establish the Kempner Institute for the Study of Natural and Artificial Intelligence.

\clearpage

{
\footnotesize
\bibliography{ref}
}

\clearpage
\appendix

\numberwithin{equation}{section}
\numberwithin{figure}{section}

\section{Proofs and Discussions}
\subsection[Proof of Lemma 1]{Proof of \Cref{lemma:bound}} \label{proof:lemma1}

We first restate the lemma using dual norm formulation. 

\paragraph[Restatement of Lemma 1]{Restatment of \Cref{lemma:bound}} \textit{Suppose a neural network classifier $F: \mathbb{R}^n \rightarrow \mathbb{R}^K$ is continuously differentiable and can be decomposed as in \Cref{eq:decomp}. Suppose $x \in \mathbb{R}^n$ is a sample input belonging to class $c$ and $\delta \in \mathbb{R}^n$ an adversarial perturbation to $x$ with error class predicted as $k$. Further, given $\frac{1}{p} + \frac{1}{q} = 1$, assume $\Vert \delta\Vert_p \leqslant R$ for some radius $R > 0$, we have}
\begin{equation}
    \Vert \delta_x\Vert_p \geqslant \frac{\left(W^{(L)}_c - W^{(L)}_k \right) \Phi(x) + b^{(L)}_c - b^{(L)}_k}{\Vert W^{(L)}_c - W^{(L)}_k\Vert_q} \cdot \frac{1}{\underset{y \in B_p(x, R)}{\max} \Vert \nabla \Phi(y)\Vert_q}\label{eq:bound}
\end{equation}

\textit{where $f(z) = W^{(L)}z + b^{(L)}$ is the last layer linear transformation, $W^{(L)}_c$ is the $c$-th row of $W^{(L)}$ treated as a row vector, and $\nabla\Phi(.)$ is the Jacobian of the feature map w.r.t. input.}

\begin{proof}
This lemma directly extends the main theorem presented in \cite{hein2017formal}, where a general classifier was considered instead. 

Define $h = f \circ \Phi$. First, $F(x)$ and $h(x)$ have the same ordering of its coordinates, since $\texttt{SoftMax}$ is a strictly monotonic transformation. More explicitly, 
\begin{equation}
    F_c(x + \delta) \leqslant F_k(x + \delta) \iff h_c(x + \delta) \leqslant h_k(x + \delta)
    \label{eq:softmax}
\end{equation}

second, by Taylor expansion to the first order in integral form, we have 

\begin{equation}
    h_k(x + \delta) = h_k(x) + \int_0^1 \langle \nabla h_k (x + t\delta), \delta \rangle \;dt, \;\; \forall k \in [K]
    \label{eq:taylor}
\end{equation}

where $\nabla h_k(.)$ is the gradient taken w.r.t. to the input, not the parameter of the neural network. Applying \Cref{eq:taylor} to \Cref{eq:softmax} on both sides, we have 

\begin{align}
    h_c(x) - h_k(x) &\leqslant \int_0^1 \langle \nabla h_k(x + t\delta) - \nabla h_c(x + t\delta), \delta  \rangle \; dt\\
    &\leqslant \Vert \delta\Vert_p \int_0^1 \Vert \nabla h_k(x + t\delta) - \nabla h_c(x + t\delta)\Vert_q \; dt \;\;\; (\text{Hölder with }\frac{1}{p} + \frac{1}{q} = 1) \\
    &\leqslant \Vert \delta\Vert_p \cdot \underset{y \in B_p(x, R)}{\max} \Vert \nabla h_k(y) - \nabla h_c(y)\Vert_q \;\;\;\;\;\qquad (R \text{ the norm of } \delta) \\
    &\leqslant \Vert \delta\Vert_p \cdot \underset{y \in B_p(x, R)}{\max} \Vert \nabla \Phi(y)\Vert_q\Vert \nabla f_k(\Phi(y)) - \nabla f_c(\Phi(y))\Vert_q \;\;\; (\text{Chain Rule}) \label{eq:prebound}
\end{align}

here we use $\Vert .\Vert_q$ to denote the vector-induced matrix norm when the input is a matrix. Using \Cref{eq:prebound}, rearranging terms, we get 

\begin{equation}
    \Vert \delta\Vert_p \geqslant \frac{h_c(x) - h_k(x)}{\underset{y \in B_p(x, R)}{\max} \Vert \nabla \Phi(y)\Vert_q\Vert \nabla f_k(...) - \nabla f_c(...)\Vert_q}
\end{equation}

note that we explicitly drop the dependence of $\nabla f$ on its input since $f$ is assumed linear. Parametrizing $f(.)$ by $W^{(L)}, b^{(L)}$ and use the fact that $h = f \circ \Phi$, we get \Cref{eq:bound}. Lastly, taking $p = q = 2$ and assuming centered data, we get \Cref{eq:bound_rewrite}.
\end{proof}

\subsection{Derivation of Feature Regularization}\label{sup:feat-reg}

In this section we derive the regularization presented in \Cref{eq:eig-ub}. Consider an explicit parameterization of $F(.; \Theta)$ as a neural network with only linear layers parametrized by $\Theta = \{W^{(1)}, b^{(1)},..., W^{(L)}, b^{(L)}\}$ and non-linear activation function $\phi(.)$ and denote $z^{l}$ as the $l$-th layer preactivation value and $D^{(l)}$ a diagonal matrix with diagonals given by $\texttt{diag}(D^{(l)}) = \phi'(z^{(l)})$, then 
\begin{equation}
    \nabla \Phi(x) = D^{(L - 1)} W^{(L - 1)}D^{(L - 2)} W^{(L - 2)} \hdots D^{(1)} W^{(1)}
\end{equation}

therefore 

\begin{align}
    \lambda_{\max}(g) &= \lambda_{\max}(D^{(L - 1)} W^{(L - 1)} \hdots D^{(1)} W^{(1)} (W^{(1)})^T (D^{(1)})^T \hdots (W^{(L - 1)})^T (D^{(L - 1)})^T) \\
    &\leqslant \prod_{l=1}^{L - 1} \lambda_{\max}(W^{(l)}(W^{(l)})^T) \cdot \lambda_{\max}(D^{(l)}(D^{(l)})^T)
\end{align}

where the last line uses the cyclic property in computing the eigenvalues recursively. Note that for common activation function the derivatives are upper-bounded by 1, and so $\lambda_{\max}(D^{(l)}(D^{(l)})^T)$ is upper bounded by 1. We have 

\begin{equation}
    \lambda_{\max}(g) \leqslant \prod_{l=1}^{L - 1} \lambda_{\max}(W^{(l)}(W^{(l)})^T) = \prod_{l=1}^{L - 1} \sigma^2_{\max}(W^{(l)})
\end{equation}

where $\sigma_{\max}$ is the largest singular value. This motivates \Cref{eq:eig-ub} by replacing the product by a sum for easier back propagation. 

\subsection[Tracking angle with fixed representations]{Tracking $\theta_x$ with fixed representations}\label{sup:last-layer}

In this section we comment on the dynamics of $\theta_x$ when we fix the neural representations. Under simplifying assumptions, we can analytically compute the last layer alignment and thus $\theta_x$ here. 

\citet{kothapalli2023neural} showed that in deep neural network the representation can demonstrate a neural collapse (NC) phenomenon where the representations form a simplex equiangular tight frame (ETF) (see also the original work by \citet{papyan2020collapse}, and theoretical work by \citet{lu2022collapse}). That is, data of the same class are sent to the same place with unit norm in the feature space and data from different class are maximally distant from each other. Assuming the ETF layout, we can analytically derive the last layer alignment under certain conditions. We present the following lemma. 

\begin{lemma}\label{lemma:alignment}
    Suppose that we have $K$ samples in the dataset and each with a different class. In this $K$-class classification with feature space $\mathbb{R}^d$ with $d > K$, assuming that representations $\{z_k\}_{k=1}^K$ reach an ETF layout, with small random Gaussian initialization for the last layer $f(z): \mathbb{R}^d \rightarrow \mathbb{R}^K: z \rightarrow W z$, gradient descent on the last layer only aligns weight vectors $W_k$, the $k-th$ row of $W$, with the neural representations of the sample from $k$-th class.
\end{lemma}

\begin{proof}
    Since we only have $K$ samples and each comes from a different class, a full-batch gradient descent has the following cross entropy loss: 

    \begin{equation}
        L(W) = \frac{1}{K}\sum_{k=1}^K \log\left(1 + \sum_{l \ne k} e^{(W_l - W_k)z_k}\right) 
    \end{equation}

    Define the composition of last layer linear map with the Softmax operation as $g = \texttt{Softmax} \circ f$. The gradient descent direction for each row of $W$ can be thus given by
    \begin{equation}
        -\frac{\partial L}{\partial W_k} = \frac{1}{K} \left[ \left(\sum_{l \ne k} g_l(z_k)\right) z_k - \sum_{l \ne k} g_k(z_l) z_l \right]
    \end{equation}

    where $g_k$ is the $k$-th output of $g$. Using small initialization we have $g(x_k) \approx [\frac{1}{K}, \frac{1}{K}, ..., \frac{1}{K}]$ close to a uniform distribution for all $k \in [K]$. This suggests that at the first step, the descent direction is given by 

    \begin{equation}
         -\frac{\partial L}{\partial W_k} \approx \frac{K - 1}{K^2} z_k - \frac{1}{K^2} \sum_{l \ne k}z_l = \frac{1}{K} z_k
    \end{equation}

    where in the last step we use $\sum_{l \ne k} z_l = - z_k$, since in an ETF layout we have $\sum_{k = 1}^K z_k = 0$.

    We may then repeat the analysis for subsequent steps. We no longer have uniform output values for each data sample, but $g_l(z_k), \forall l \ne k$ would be roughly equal by symmetry and small, so that the descent direction continues to pull $W_k$ towards $z_k$ albeit with a smaller strength.
\end{proof}

We demonstrate this prediction in a toy 2D task: suppose 3 samples in 2 dimensions are of different classes and they form a simplex ETF layout. This means they are on the unit circle and are exactly 120 degrees away from each other. Starting with Guassian initialization, the last layer vectors gradually align with the representation vectors under a full-batch gradient descent training with learning rate 0.01. We demonstrate the alignment process for small initialization (std=0.001) in \Cref{fig:etf-alignment-small} and large initialization (std=0.1) in \Cref{fig:etf-alignment-large}.

\begin{figure}[t]
    \centering
    \begin{subfigure}{0.32\textwidth}
        \includegraphics[width=\linewidth]{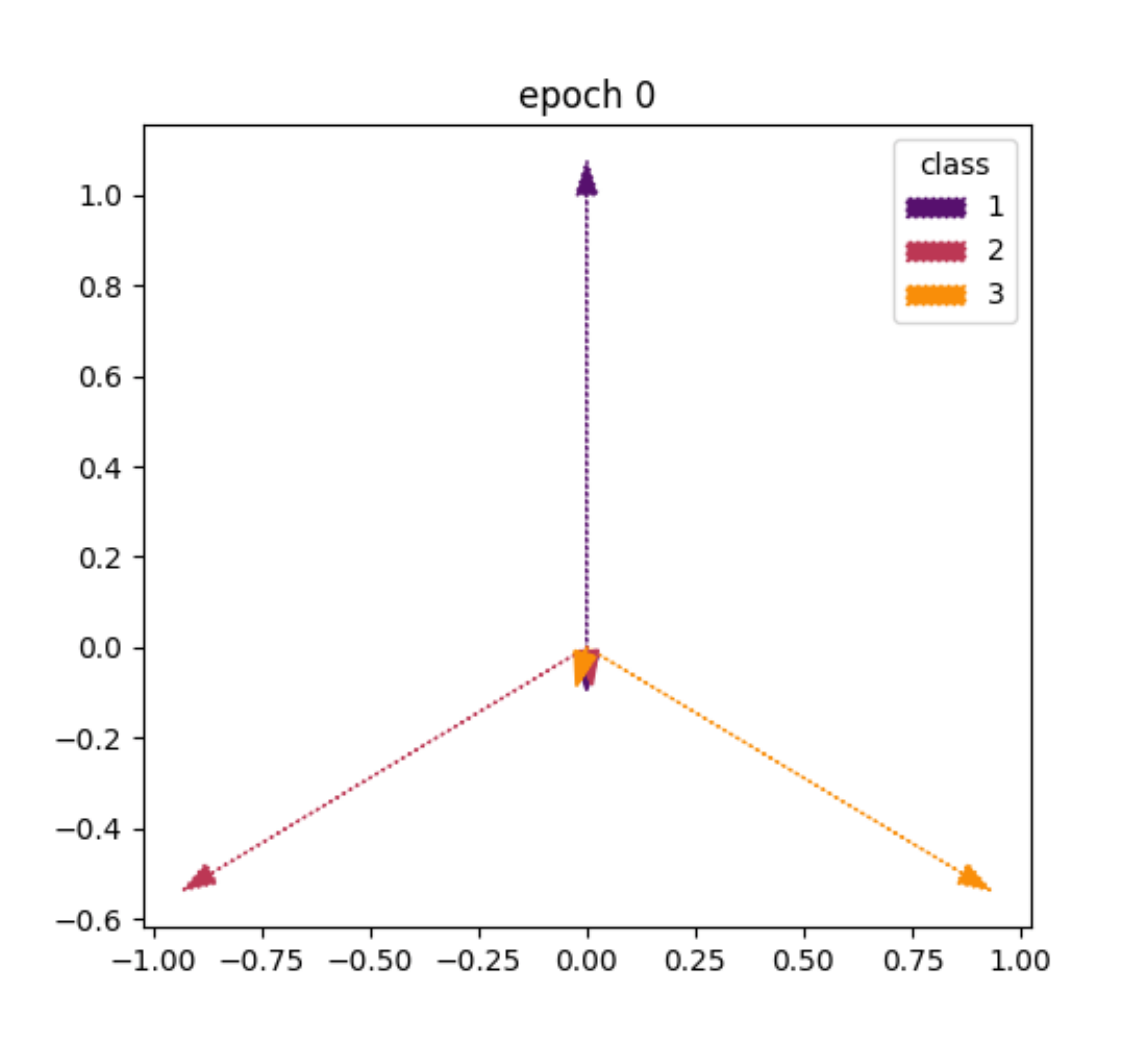}
    \end{subfigure} \hfill
    \begin{subfigure}{0.32\textwidth}
        \includegraphics[width=\linewidth]{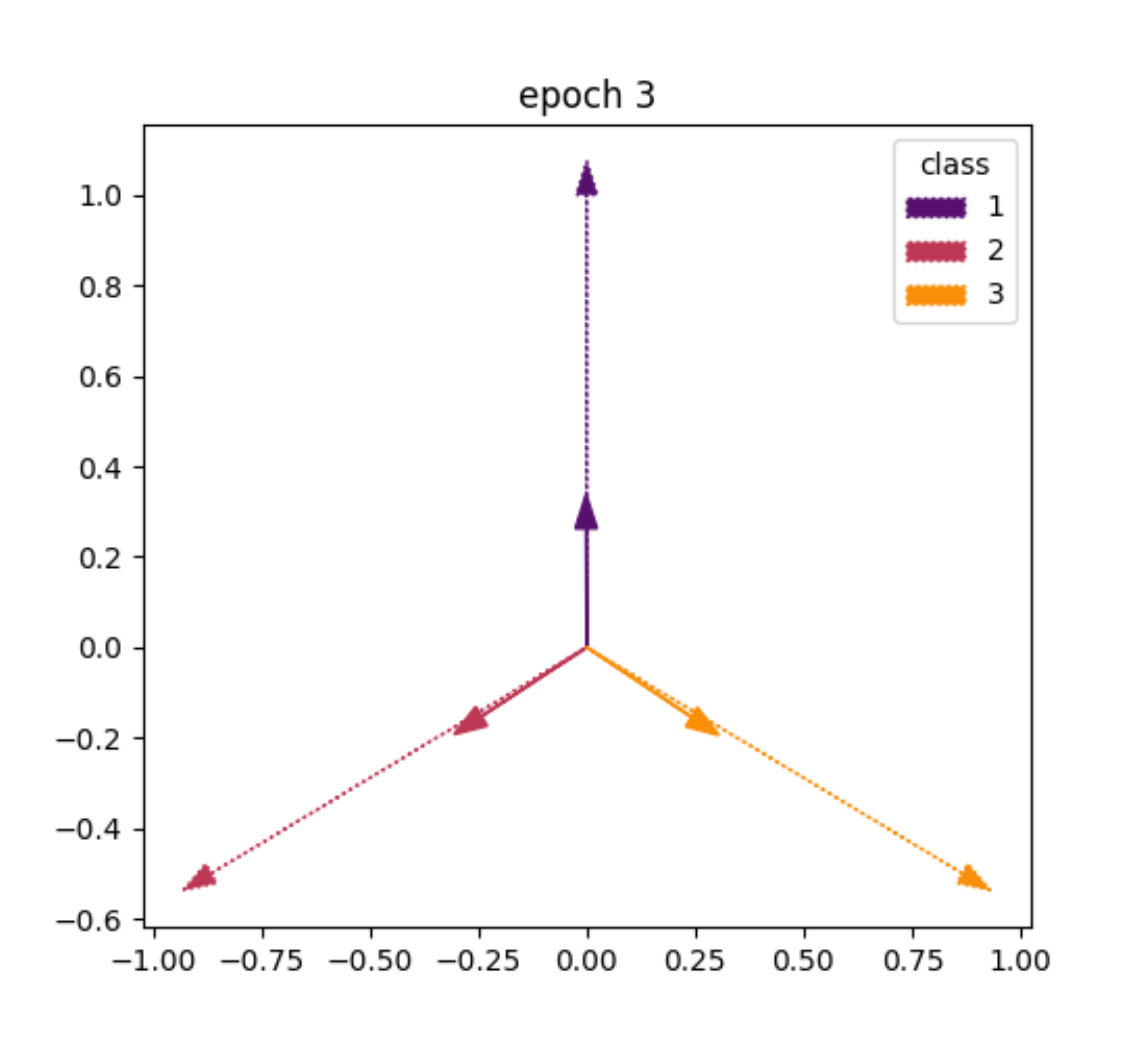}
    \end{subfigure} \hfill
    \begin{subfigure}{0.32\textwidth}
        \includegraphics[width=\linewidth]{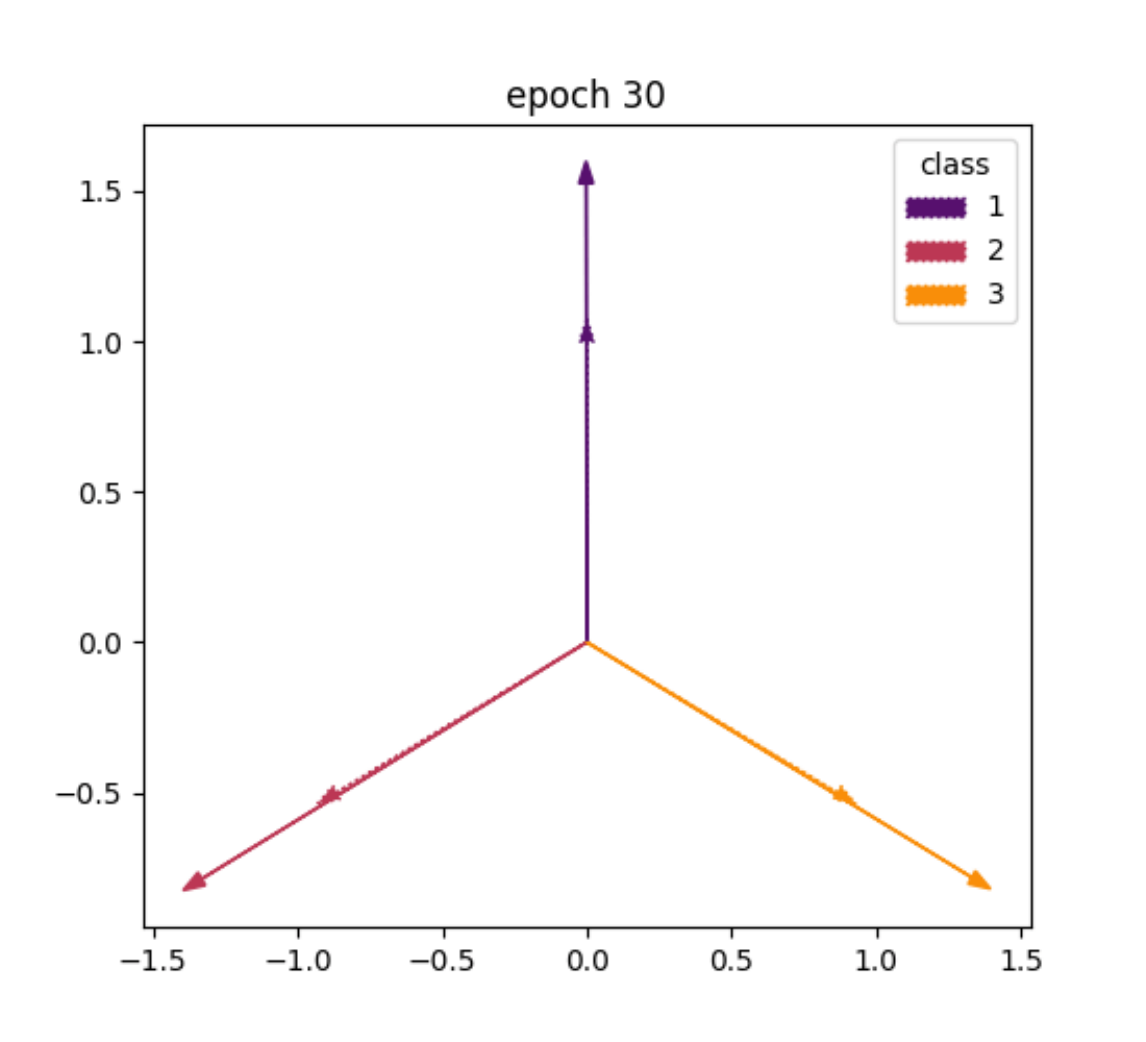}
    \end{subfigure}
    \caption{Last Layer Alignment with $\text{std}=0.001$ initialization}
    \label{fig:etf-alignment-small}
\end{figure}

\begin{figure}[t]
    \centering
    \begin{subfigure}{0.32\textwidth}
        \includegraphics[width=\linewidth]{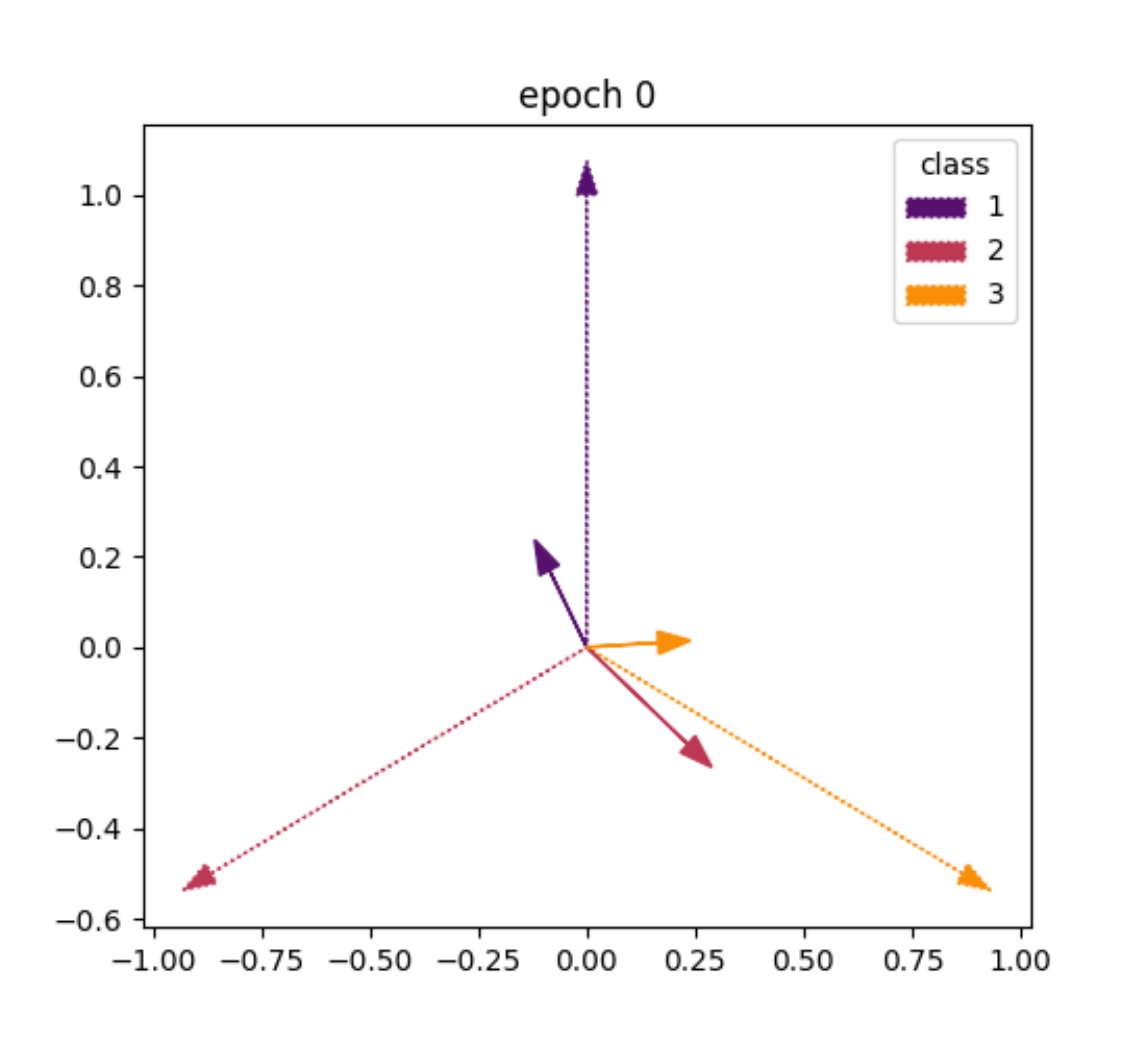}
    \end{subfigure} \hfill
    \begin{subfigure}{0.32\textwidth}
        \includegraphics[width=\linewidth]{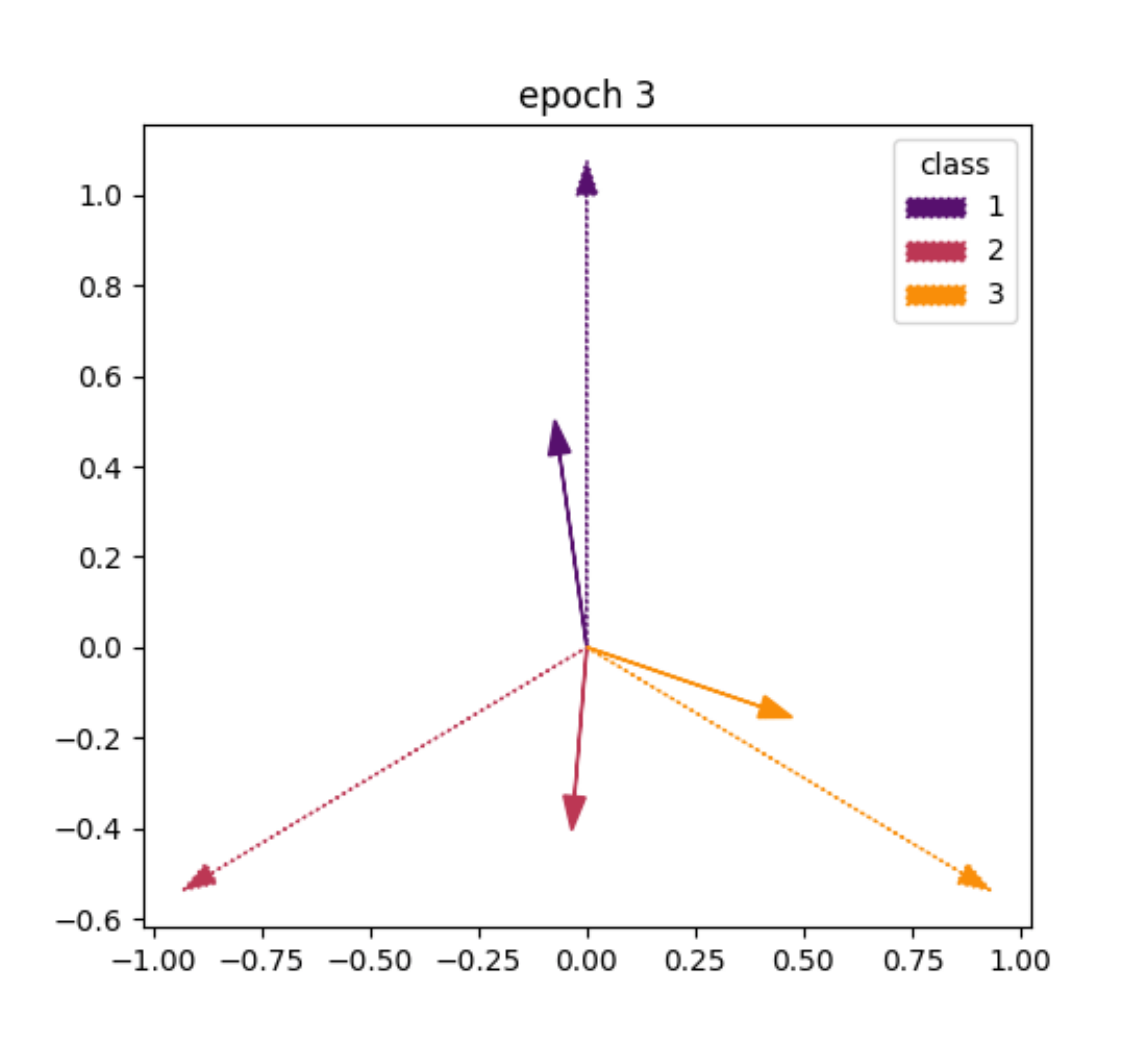}
    \end{subfigure} \hfill
    \begin{subfigure}{0.32\textwidth}
        \includegraphics[width=\linewidth]{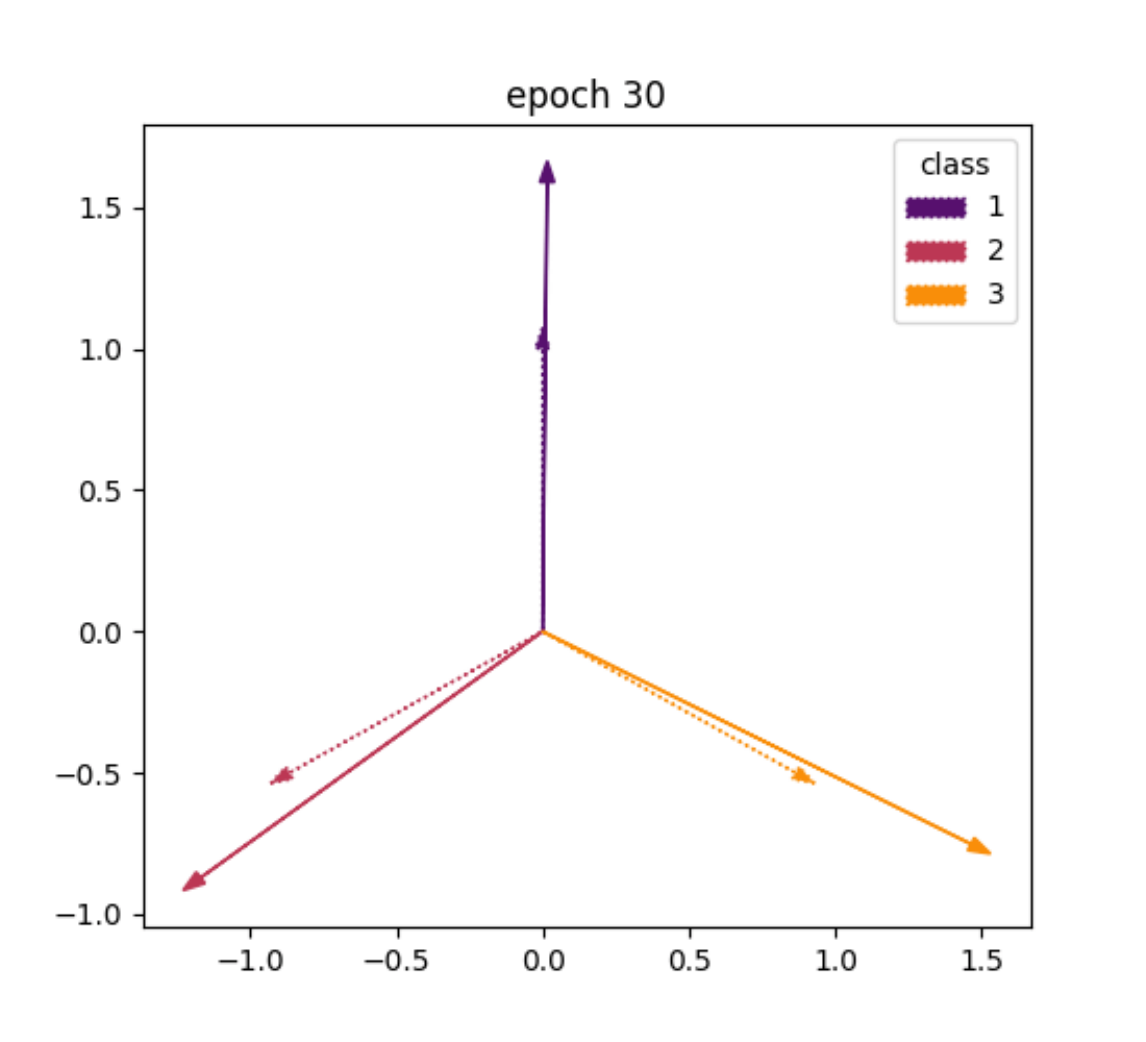}
    \end{subfigure}
    \caption{Last Layer Alignment with $\text{std}=0.1$ initialization}
    \label{fig:etf-alignment-large}
\end{figure}

We could extend the results to arbitrary number of data samples with balanced class label distributions as well. With the same assumptions, we can write down $\theta_x$ in simple cases. 

\begin{lemma}\label{lemma:thetax_analytic}
    With the same assumption in \Cref{lemma:alignment}, and suppose $x$ belongs to class $c$ with representation $z_c$ and $z_k$ is one feature representation from any other class, we have 

    \begin{equation}
        \theta_x \approx \frac{1 - z_k^T z_c}{\Vert z_c - z_k \Vert_2}
        \label{eq:thetax_analytic}
    \end{equation}

    whose value is the same for all $k \ne c$ due to the symmetry in the ETF layout.
\end{lemma}

\begin{proof}
    Using \Cref{lemma:alignment}, $W_k \approx c z_k, c > 0$, with the same scalar $c$ for all class $k$. Plugging in the definition of $\theta_x$ we get \Cref{eq:thetax_analytic}.
\end{proof}

If we relax the ETF layout assumption, one may still write down the gradient and analytically track the gradient flow and get expressions of $W_k$ at convergence, but in that case $W_k$ do no align with $z_k$ due to a lack of symmetry in the feature representations, and subsequently $\theta_x$ for different classes will be different. 

\subsection[Tracking angles in joint training]{Tracking $\theta_x$ and $\Vert \Phi(x) \Vert_2$ during joint training}

In general, if we jointly train feature representation and last layer, we do not have the nice analytic prediction as in \Cref{lemma:alignment} and \Cref{lemma:thetax_analytic}. We could still empirically track the two quantities. In simple cases we see that compared to initialization, both $\Vert \phi(x) \Vert_2$ and $\theta_x$ expand when training with cross entropy loss. See \Cref{fig:xor_demo} for cross entropy loss training on 2D XOR dataset and \Cref{fig:mnist_demo} for MNIST data. This solidifies our argument that from a feature representation perspective, to bump the lower bound on adversarial robustness, one would like to control the $\Vert \nabla \Phi(.) \Vert_2$ term. 

\begin{figure}[t]
    \centering
    \begin{subfigure}{0.99\linewidth}
        \centering
        \includegraphics[width=\linewidth]{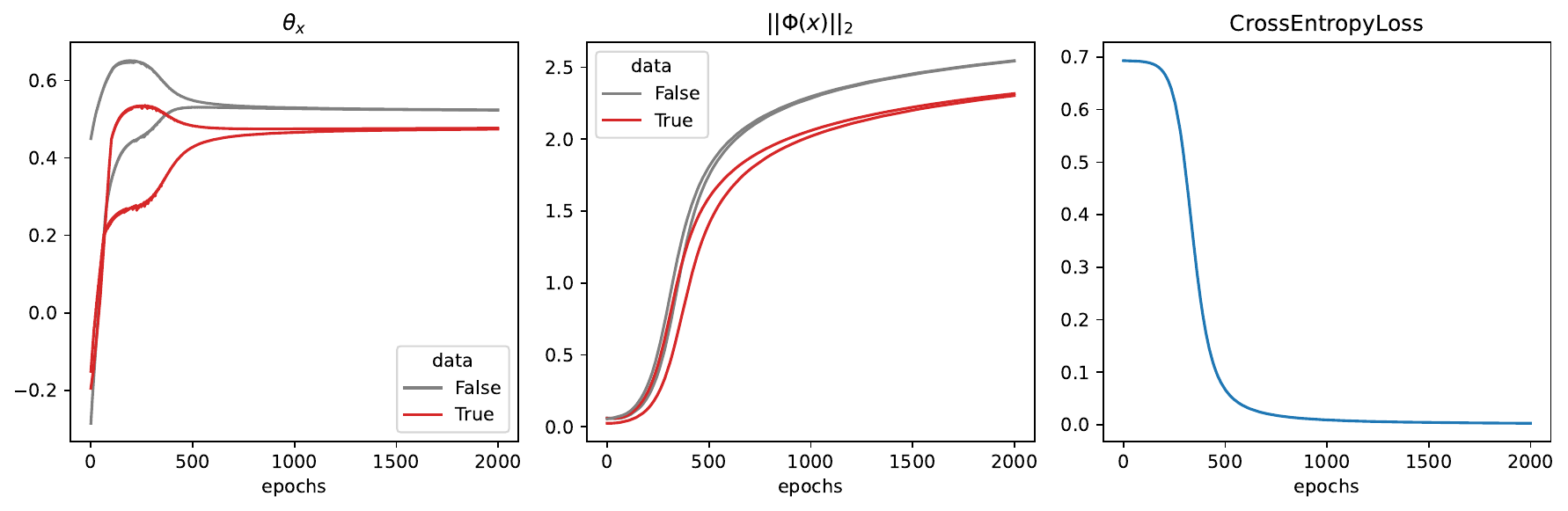}
        \caption{$\theta_x, \Vert \Phi(x)\Vert_2,$ and cross entropy loss over the course of training}
    \end{subfigure} \\
    \begin{subfigure}{0.99\linewidth}
        \centering
        \includegraphics[width=\linewidth]{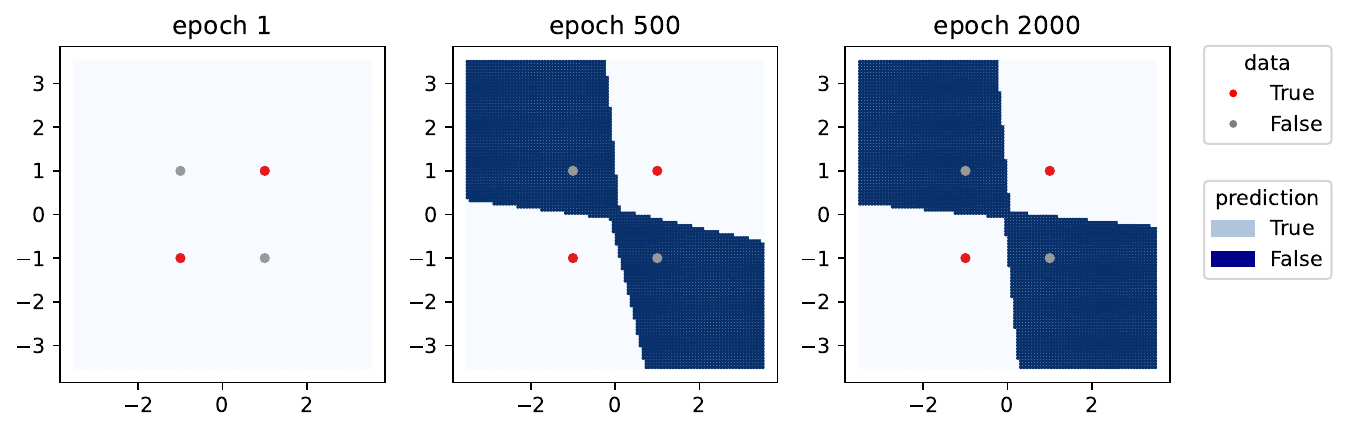}
        \caption{Snapshot of decision boundary at different training epochs}
    \end{subfigure}  
    \caption{In this simple example, we train XOR task on a single-hidden-layer neural network with 20 hidden units using full batch gradient descent and Gaussian initialization with zero mean and 0.01 std. Compared to initialization, $\theta_x$ increased and $\Vert \Phi(x)\Vert_2$ expanded over the course of training.}
    \label{fig:xor_demo}
\end{figure}

\begin{figure}[t]
    \centering
    \begin{subfigure}{0.99\linewidth}
        \centering
        \includegraphics[width=\linewidth]{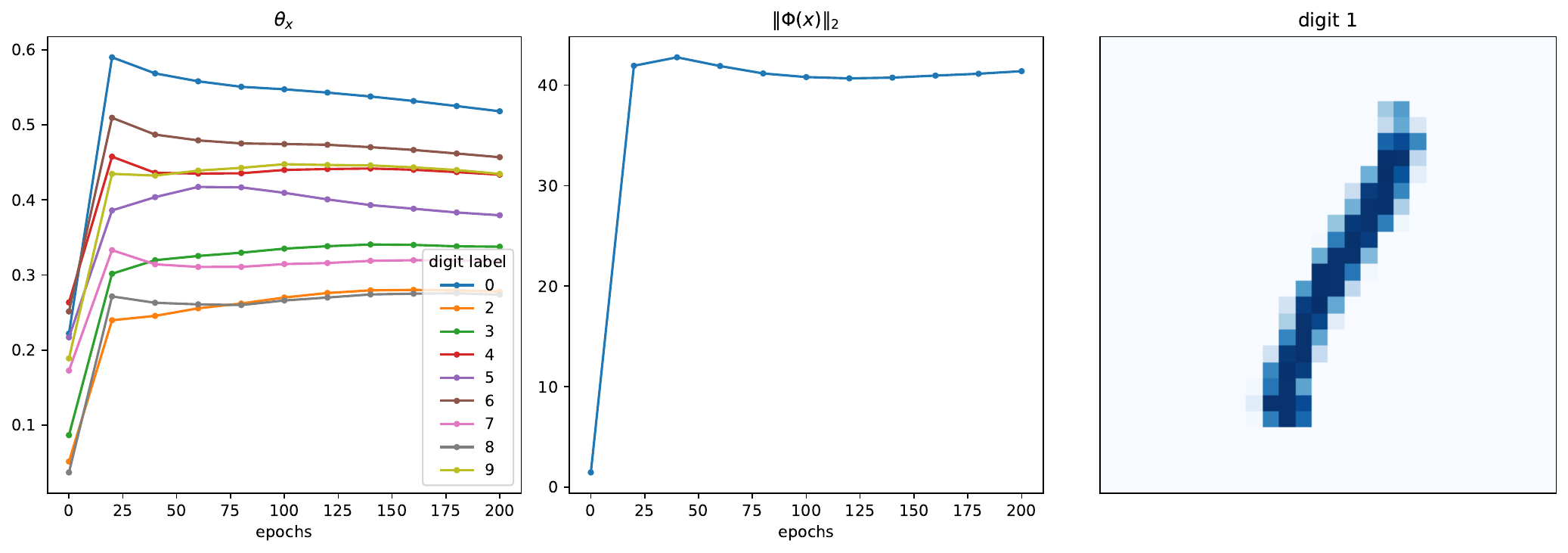}
    \end{subfigure} \\
    \begin{subfigure}{0.99\linewidth}
        \centering
        \includegraphics[width=\linewidth]{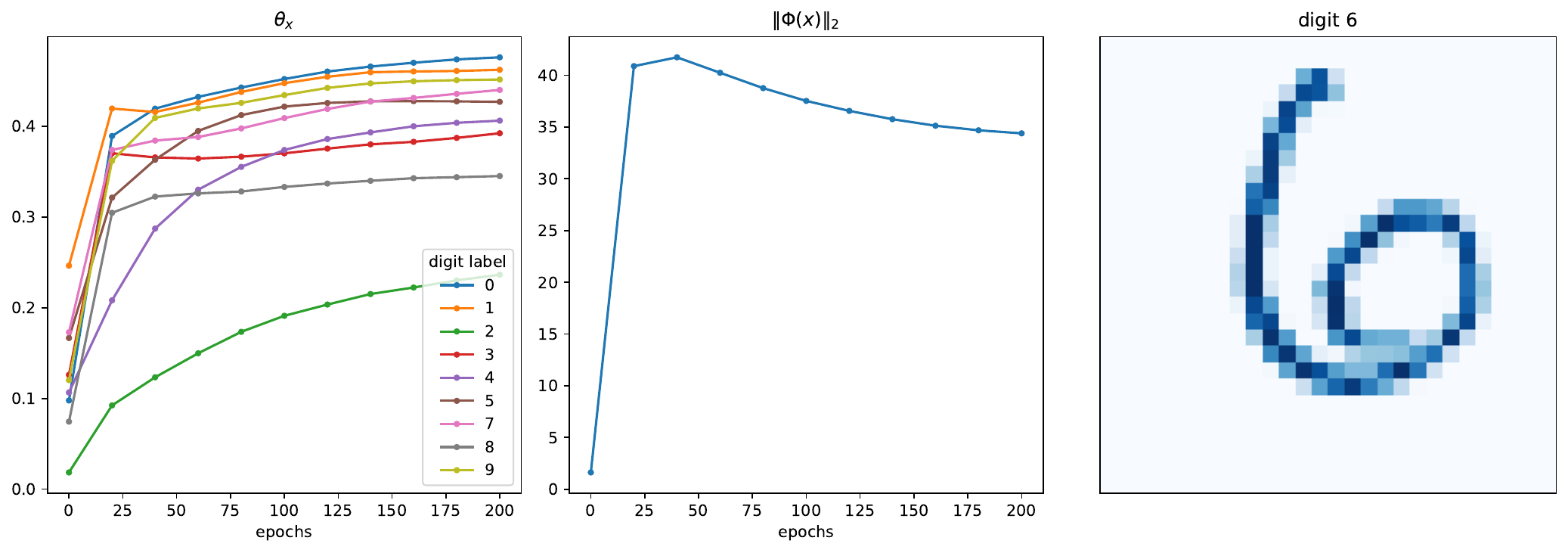}
    \end{subfigure}
    \caption{$\theta_x, \Vert \Phi(x)\Vert_2$ for a test sample of class 1 (top panel) and class 6 (bottom panel). In MNIST setting with small initialization (gaussian mean 0 std 0.01) and normal cross entropy loss training, we observe expansion of $\theta_x$ and $\Vert \Phi(x)\Vert_2$ over training.}
    \label{fig:mnist_demo}
\end{figure}

\section{Linearization of a Multi-channel Convolution Layer}\label{sup:linearize_cnn}

This section provides a brief description of the linearization of multi-channel convolution layer and subsequently getting the maximum eigenvalue of this linear operation. More details and proofs can be found in \citep{sedghi2018singular,senderovich2022towards}.
\subsection[Construction of the linear map K]{Construction of the linear map $\tilde{\mathcal{K}}$}

A periodic 2D convolution operation can be considered a linear transformation on the vectorized input, and the weights are constructed from the filters. Consider $X \in \mathbb{R}^{\text{c}_{\text{in}} \times n \times n}$ an input image to a convolution layer with $\text{c}_{\text{in}}$ numbers of input channels and height/width given by $n$. A multichannel filter $\mathcal{K} \in \mathbb{R}^{\text{c}_{\text{out}} \times \text{c}_{\text{in}} \times k \times k}$ with stride $s$ with  $\text{c}_{\text{out}}$ number of output channels and kernel size $k$ can be rewritten into a matrix $\mathcal{\tilde{K}} \in \mathbb{R}^{c_{\text{out}}n_{\text{out}}^2 \times c_\text{in}n^2}$ such that 
\begin{equation}
    \operatorname{Vec}(\texttt{Conv2D}(X)) = \mathcal{\tilde{K}} \; \operatorname{Vec}(X)
\end{equation}

where $\operatorname{Vec}(.)$ is a row-major reshaping of $X$ (i.e. the default behavior of calling \texttt{.flatten()} in \texttt{NumPy} and \texttt{PyTorch}), and $n_{\text{out}}$ the output height/width given by $n_{\text{out}} = \lfloor \frac{n - 1}{s} + 1 \rfloor$. 

The transformation $\tilde{\mathcal{K}}$ consists of $n^2 \times n^2$ blocks of doubly circulant matrix, and each of the doubly circulant matrix contains data that come from appropriately slicing the zero-padded filter $\mathcal{K}$ that matches with the same shape of the 2D input image. It could be validated that the singular values of $\tilde{\mathcal{K}}$ is the union of all singular values of 2D FFT-transformed blocks of the appropriate slicing, so that to compute the top singular value of $\tilde{\mathcal{K}}$, one does not need to construct $\tilde{\mathcal{K}}$ itself but instead should record the top singular values of FFT-transformed slices and take the max of all these maximum singular values, which saves substantial computational resources. 

In Code Block \ref{convolution_code} we present \texttt{PyTorch} code that is modified from theorem 2 of \cite{senderovich2022towards} for computing the square of top singular value of $\tilde{\mathcal{K}}$.

\begin{codeblk}[t]
\begin{python}
# import packages
import torch
from torch.nn.functional import pad

# function body
def get_multi_channel_top_eigval_with_stride(
    kernel: torch.Tensor, h: int, w: int, stride: int
) -> torch.Tensor:
    """
    compute top eigen value of a convolution layer
    * code tested only for even n and stride = 1 or 2.

    :param kernel: the conv2d kernel, with shape (c_out, c_in, k, k)
    :param h: the image height
    :param w: the image width
    :param stride: the stride of convolution
    :return the top singular value for conv layer
    """
    # pad zeros to the kernel to make the same shape as input
    c_out, c_in, k_h, k_w = kernel.shape
    pad_height = h - k_h
    pad_width = w - k_w
    kernel_pad = pad(kernel, (0, pad_height, 0, pad_width), mode="constant", value=0)
    str_shape_height, str_shape_width = h // stride, w // stride

    # downsample the kernel
    transforms = torch.zeros(
        (c_out, c_in, stride**2, str_shape_height, str_shape_width)
    ).to(kernel.device)
    for i in range(stride):
        for j in range(stride):
            transforms[:, :, i * stride + j, :, :] = kernel_pad[
                :, :, i::stride, j::stride
            ]

    # batch fft2
    transforms = torch.fft.fft2(transforms)
    transforms = transforms.reshape(c_out, -1, str_shape_height, str_shape_width)

    # reorg (h // stride, w // stride, c_out, c_in * stride^2)
    P = transforms.permute(2, 3, 0, 1)

    # compute singular value squared
    # by computing eigenvalues of KK^T or K^TK, whichever is faster
    eigvals = torch.linalg.eigvalsh(
        torch.einsum("...ij,...kj->...ik", torch.conj(P), P)
        if P.shape[3] > P.shape[2]
        else torch.einsum("...ji,...jk->...ik", torch.conj(P), P)
    )

    # keep top eigenvalue only
    top_eig = eigvals.max()
    return top_eig
\end{python}
\caption{Python code for computing the top singular value of the operator form of a convolutional layer}
\label{convolution_code}
\end{codeblk}

\subsection{Speeding Up Top Eigenvalue Computation: Power Iteration Across Parameter Updates}

Since here we are only interested in the maximal eigenvalue, we could use batched power iteration (algorithm described in \Cref{sup:power-iteration}) to jointly compute the top eigenvalues only for 2D FFT-transformed blocks. Both 2D FFT transformation and batch eigenvalue update are GPU friendly. Furthermore, since between consecutive parameter updates, the change in filter map is likely small, iterations in the power method can be amortized across parameter updates. We empirically find that the outcome of conducting a batch power iteration update to the top eigenvalues every 20 parameter updates have little difference in performance compared to computing the exact eigenvalue before each parameter update. See \Cref{sup:deep_conv} for the results.

The amortization across parameter update or across epochs trick has been noticed previously in \citet{yoshida2017spectral}, but was only applied to fully-connected layers. 

\subsection{Implication of the Dependence on Input Dimension}

Although a 2D convolution operator applies to images with arbitrary size and channels, the top singular value of the linearized map depends on the size and channels. To impose spectral regularization to convolution layers, it is thus recommended to use the size and channels of the test images for regularization. The impact of regularizing on one set of images and testing on another with different shapes and channels remains to be explored.

\section{Power Iteration to Compute the Top Singular Value of a Linear Map}\label{sup:power-iteration}

Power iteration is an iterative algorithm to compute the eigenspectrum of a diagonalizable matrix. We can extend the algorithm to compute the largest singular value of a complex-valued matrix. The algorithm is described in \Cref{algo:power-iter}. In practice, $N=1$ is enough, since the parameter moves slowly as we train the cross entropy classification loss. 

\begin{algorithm}[t]
    \caption{Power Iteration for Top Singular Value Squared}\label{algo:power-iter}
    \begin{algorithmic}
    \Require $M \in \mathbb{C}^{m \times n}, N \in \mathbb{N}$ \Comment{N number of iterations}
    \Ensure $\lambda = \sigma_1(M)^2$
    \State $v = v_0$ \Comment{Randomly Generated or from previous training iterates}
    \State $v \gets \frac{v}{\Vert v\Vert_2}$ \Comment{Normalize}
    \State $i \gets 0$
    \While{$i < N$}
        \State $u \gets M v$
        \State $u \gets \frac{u}{\Vert u\Vert_2}$
        \State $v \gets M^* u$ \Comment{Conjugate Transpose}
        \State $v \gets \frac{v}{\Vert v\Vert_2}$
        \State $i \gets i + 1$
    \EndWhile

    \State $p = M v$ \Comment{$Mv_1 = \sigma_1 u_1$}
    \State $\lambda = \Vert p\Vert_2^2$
    \end{algorithmic}
\end{algorithm}

\section[Finding Adversarial Distances by Tangent Attack ]{Finding Adversarial Distances by Tangent Attack \cite{ma2021finding}}\label{sup:ta}

The Tangent Attack algorithm proposed by Ma et al. \cite{ma2021finding} provides a good heuristic for finding adversarial distances in the black-box setting. Here, we briefly sketch how this algoritm operates; we refer the reader to \cite{ma2021finding} for details. 

Given a correctly classified sample $x$, we initialize the algorithm by adding a fixed number of Gaussian perturbations with predetermined standard deviations that adapts to the input dimension. Among all perturbed samples, we keep ones that were classified differently by the neural network classifier and select the one with the minimum $l_2$ distance to the original sample $x$. We then run a binary search along the line segment from $x$ to that sample to locate a point that is on the decision boundary. We call this point $x_0$ as our initial guess to the adversarial sample to $x$. 

Next we iteratively shrink the distance between $x_t, t \in \{0, 1, ..., T\}$ and $x$, where $T$ is a predefined maximum number of updates to the adversarial guesses so that at the end of the algorithm $x_T$ is considered the adversarial sample to $x$ and the adversarial distance $\delta_x = \Vert x_T - x\Vert_2$. The update is done by performing the following three key steps in sequence: at each $t \in \{0, 1, ..., T - 1\}$
\begin{enumerate}
    \item \textbf{estimate a normal direction} to the decision boundary, pointing to the adversarial region: we take local perturbation to $x_t$. Based on the prediction on these perturbations, averaging the vectors that give adversarial prediction gives us an estimate to the normal direction;
    \item \textbf{find the tangent point} in the 2D plane generated by $x$, $x_t$, and the normal direction. construct a hemisphere in the direction of the normal vector with a predefined small radius, find the tangent plane to the hemisphere that passes through $x$ and locate the tangent point $k$. This step is done by analytic geometry and a closed form update can be analytically derived. 
    \item \textbf{conduct a binary search} along the line segment from $x$ to $k$, get a sample on the decision boundary and assign that to $x_{t+1}$. In this way, $x_{t+1}$ is a valid adversarial sample with a different prediction from $x$ but is closer to $x$ than $x_t$.
\end{enumerate}

In our experiments, we use $T = 40$ throughout. Other hyperparameters such as the radius of hemisphere and the number of local perturbations for normal direction estimation follows identically from the hemisphere implementation in \cite{ma2021finding}.

\section{Experiment Details}\label{sup:exp_details}

Code to reproduce all experiments is freely available on GitHub.\footnote{\url{https://github.com/Pehlevan-Group/rep-spectral}} All experiments in this paper were run on the Harvard FASRC Cannon cluster supported, using NVIDIA A100 40GB GPUs. Experiments reported in the main text required less than 240 GPU-hours. 

Our code base is adapted from various publicly available ones, including \texttt{TangentAttack} \footnote{\url{https://github.com/machanic/TangentAttack}} with an Apache V2 license for evaluating the adversarial distances, \texttt{FixRes}\citep{touvron2019FixRes} \footnote{\url{https://github.com/facebookresearch/FixRes}} with a CC BY-NC 4.0 license for multi-GPU training Resnet50, \texttt{SimCLR} \footnote{\url{https://github.com/sthalles/SimCLR}} with an MIT license for SimCLR model data loading and evaluation, \texttt{BarlowTwins} \footnote{\url{https://github.com/facebookresearch/barlowtwins/tree/main}} with an MIT license for BarlowTwins data loading and evaluation, \texttt{practical\_svd\_conv} \footnote{\url{https://github.com/WhiteTeaDragon/practical_svd_conv}} with a BSD-3-Clause license for efficient computations of top singular value of 2D convolution layers, and lastly \texttt{nn\_curvature} \footnote{\url{https://github.com/Pehlevan-Group/nn_curvature}} with an MIT license for volume element computations.

Our \texttt{Python} code also uses some common publicly available packages, including \texttt{NumPy} \citep{2020NumPy-Array} with a BSD license, \texttt{Matplotlib} \citep{Hunter:2007} with a BSD license, \texttt{Pandas} \citep{mckinney2010data} under a BSD license, \texttt{scikit-learn} \citep{scikit-learn} with a BSD license, \texttt{PyTorch} \citep{paszke2017automatic} under a modified BSD license, \texttt{tqdm} with an MIT license, and \texttt{toml} with an MIT license. 

Data used in the project include MNIST \cite{lecun2010mnist}, CIFAR-10 \cite{krizhevsky2009cifar}, ImageNet-1K \cite{deng2009imagenet}, Stanford Dog \citep{KhoslaYaoJayadevaprakashFeiFei_FGVC2011}, Oxford Flowers \citep{Nilsback08}, and MIT indoor \citep{quattoni2009recognizing}.

\subsection{Shallow Network}\label{sup:shallow}

To train shallow network for clean and noisy XOR with 8 hidden units, we train for 15000 epochs using full batch GD with 1.0 learning rate, 0.9 momentum, zero or 1e-4 weight decay, the same regularization strength ($\gamma = 0.0001$), and a regularization burnin period of 10500 epochs (70\% unregularized training + 30\% regularized training), all models could classify all 4 points correctly but demonstrate vastly different decision landscapes.

For shallow network, there are only two connection layers. The difference between \citet{yoshida2017spectral}'s \texttt{ll-spectral} regularization and our \texttt{rep-spectral} regularization is heuristically most pronounced in this case, since here \texttt{rep-spectral} is regularizing only half of the layers that \texttt{ll-spectral} is regularizing. 

In clean XOR training, by plotting the matrix 2-norm of the connection weights in each layer through respective regularization, we found a striking contrast between the two. By imposing the same regularization strength for \texttt{rep-spectral} and \texttt{ll-spectral}, \texttt{rep-spectral} is able to control effectively the weight norm of the feature layer with an expansion of weight norm in the last layer; \texttt{ll-spectral} regularization is the exact opposite that it fails to control the weight norm in the feature layer. This is shown in the \Cref{fig:weight_norm}.

\begin{figure}[t]
    \centering
    \includegraphics[width=5in]{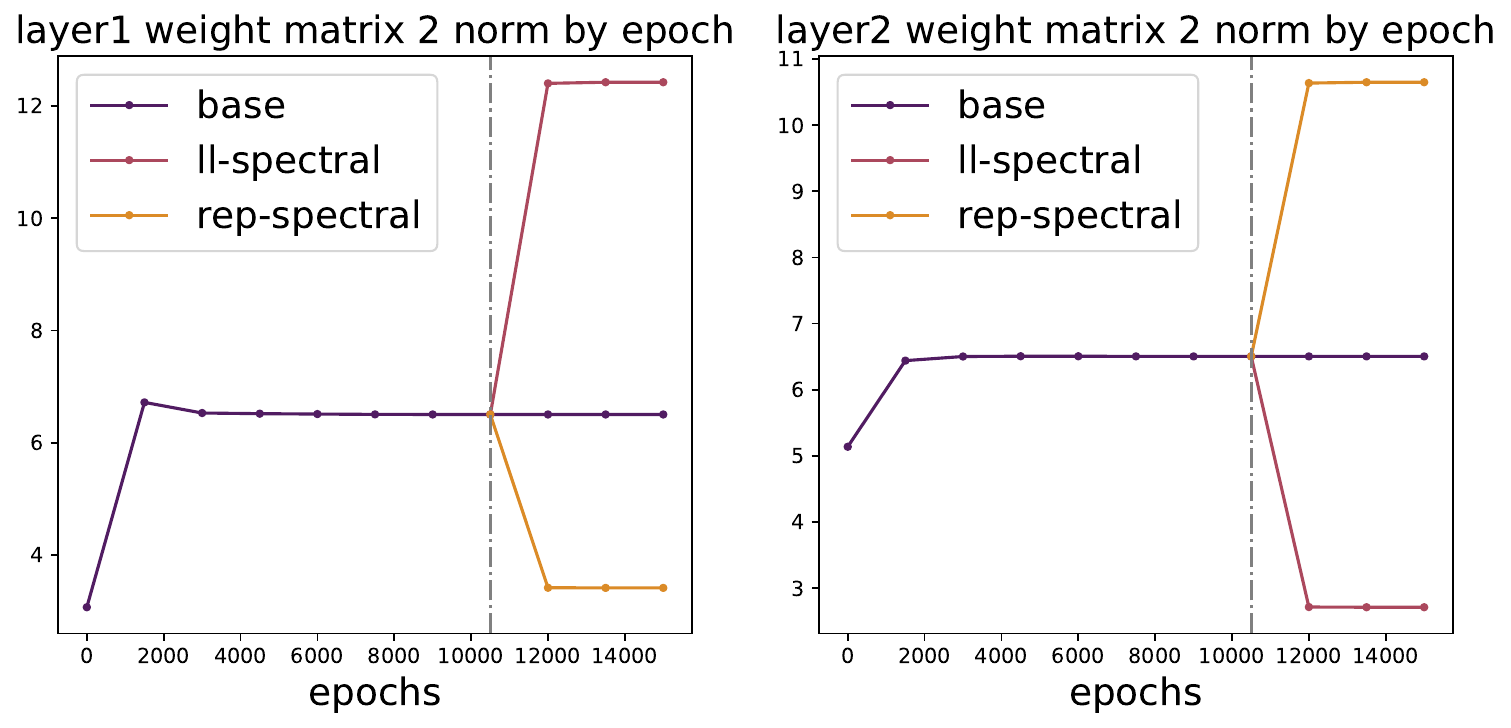}
    \caption{\texttt{ll-spectral} regularization and \texttt{rep-spectral} regularization weight norm change over training clean XOR data shown in \Cref{fig:xor_clean}. At epoch 10500 we turn on the adversarial regularization, before which there is only crossentropy loss.}
    \label{fig:weight_norm}
\end{figure}

\begin{figure}
        \centering
        \includegraphics[width=3in]{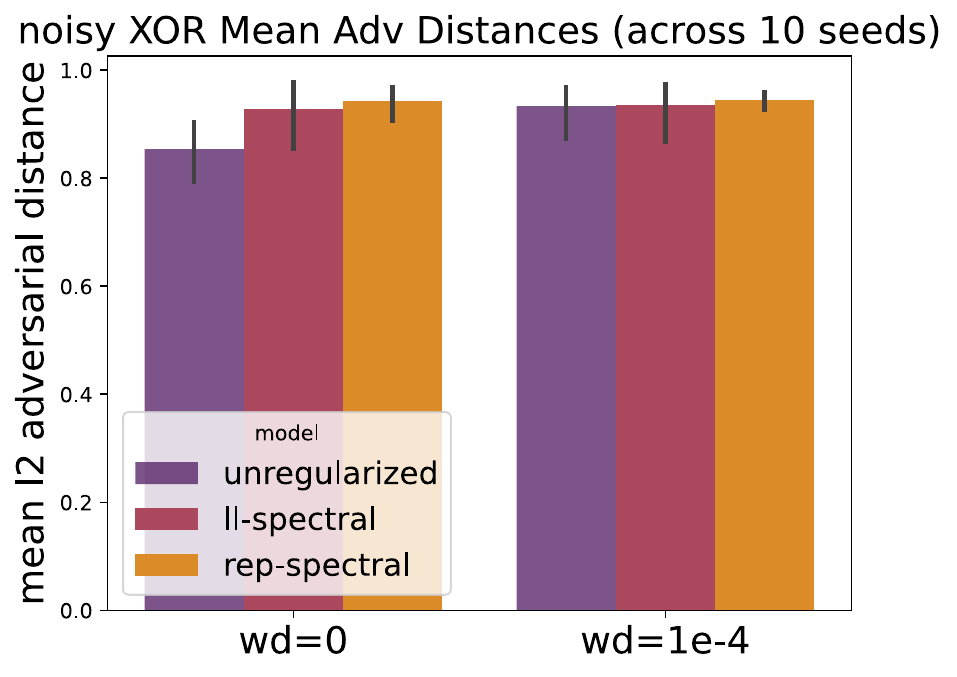}
        \label{figure:xor_noisy_seed}
    \caption{average $\Delta_x$ found by TA across 10 different seed for clean XOR (left) and noisy XOR (right) with and without weight decay of 1e-4. The error bar shows plus and minus one standard deviation across 10 seeds around the mean. This behavior is consistent across 10 different random seeds}
\end{figure}

In settings where the last layer is discarded and retrained such as transfer and self-supervised learning, \texttt{rep-spectral} method is preferred due to its attention to feature layer weight norm rather than just the last layer.

\begin{figure}[t]
    \centering
    \includegraphics[width=0.9\linewidth]{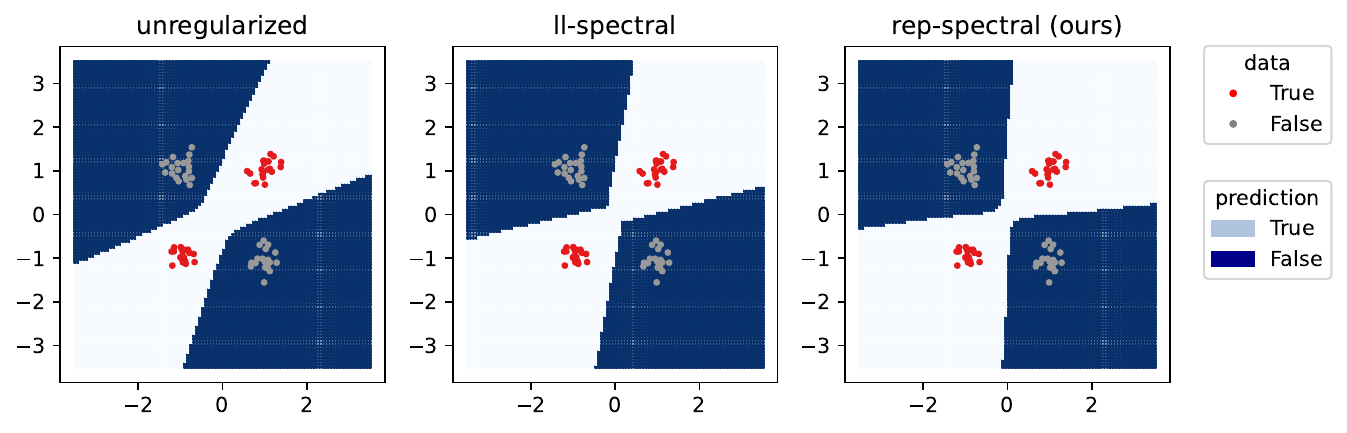}
    \caption{Noisy XOR decision boundaries}
    \label{fig:xor_noisy}
\end{figure}

To train shallow network with 784 hidden units on MNIST dataset, we train for 200 epochs using SGD with batch size 1024, learning rate 0.1, momentum 0.9, weight decay 1e-4, $\gamma = 0.001$, and a regularization burn-in period of 160 epochs (80\% unregularized training + 20\% regularized training). The update of the eigen directions are done through one power iteration every parameter update. 

To testify the robustness of representations, we retrain a new linear head using multilogisitc regression with a $l_2$ regularization parameter $1$ (i.e. the default setting when applying \texttt{sklearn.linear\_model.LogisticRegression}) trained on the feature representations of the same train samples. To make decision of a test sample, we get the feature representation through the feature map and then apply multilogistic regression trained to get decisions.

\begin{figure}[t]
    \centering
    \begin{subfigure}{\textwidth}
        \centering
        \includegraphics[width=\linewidth]{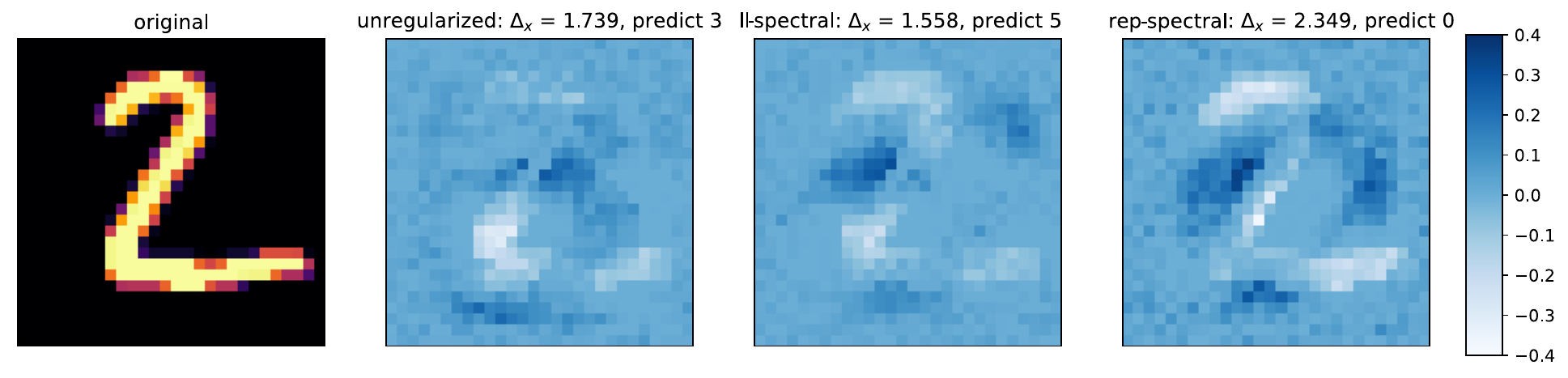}
    \end{subfigure} \\
    \begin{subfigure}{\textwidth}
        \centering
        \includegraphics[width=\linewidth]{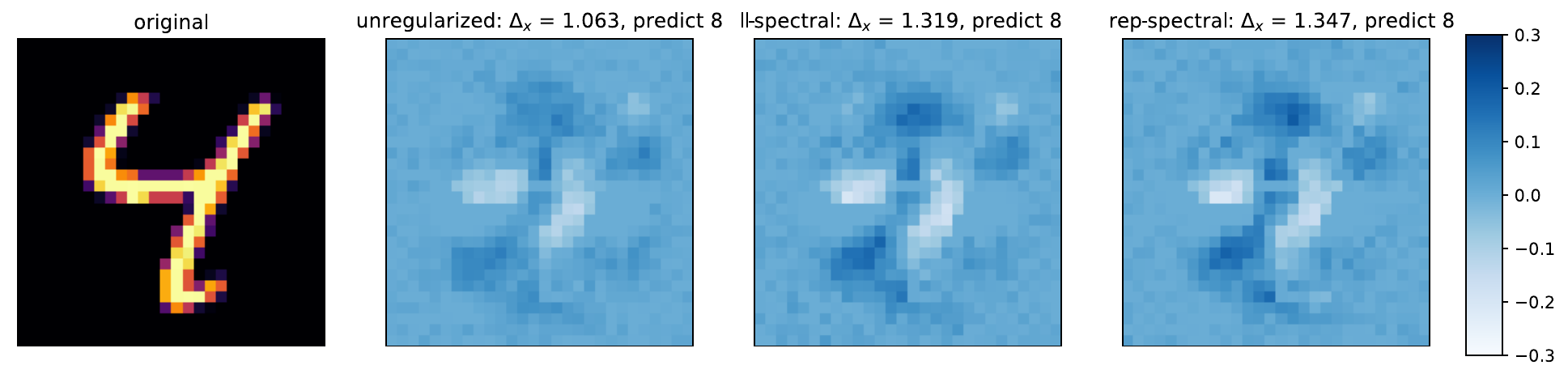}
    \end{subfigure} \\
    \begin{subfigure}{\textwidth}
        \centering
        \includegraphics[width=\linewidth]{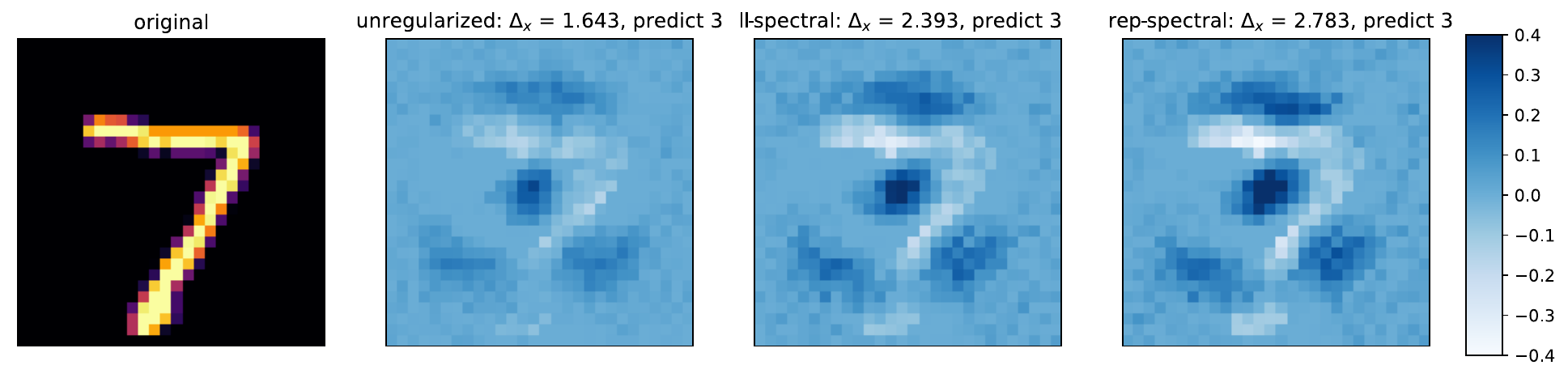}
    \end{subfigure} \\
    \caption{MNIST Test Images and the perturbation $\delta_x$ found by Tangent Attack \citep{ma2021finding} for three methods. The adversarial distances $\Delta_x$ and the adversarial predictions are reported on the titles. In general, our method encourages larger adversarial distances compared to other methods.}
    \label{fig:mnist-perturbation}
\end{figure}

\newpage

\subsection{Deep Convolutional Networks}\label{sup:deep_conv}

We train for 200 epochs using SGD with batch size 1024, learning rate 0.01, momentum 0.9, weight decay 1e-4, $\gamma = 0.01$, and an adversarial regularization burnin period of 160 epochs (80\% unregularized training + 20\% regularized training). To alleviate regularization computational cost, we only update the eigenvector direction using the power iteration once every 24 parameter updates. 

\begin{figure}
    \centering
    \begin{subfigure}{1.5in}
        \centering
        \includegraphics[width=\linewidth]{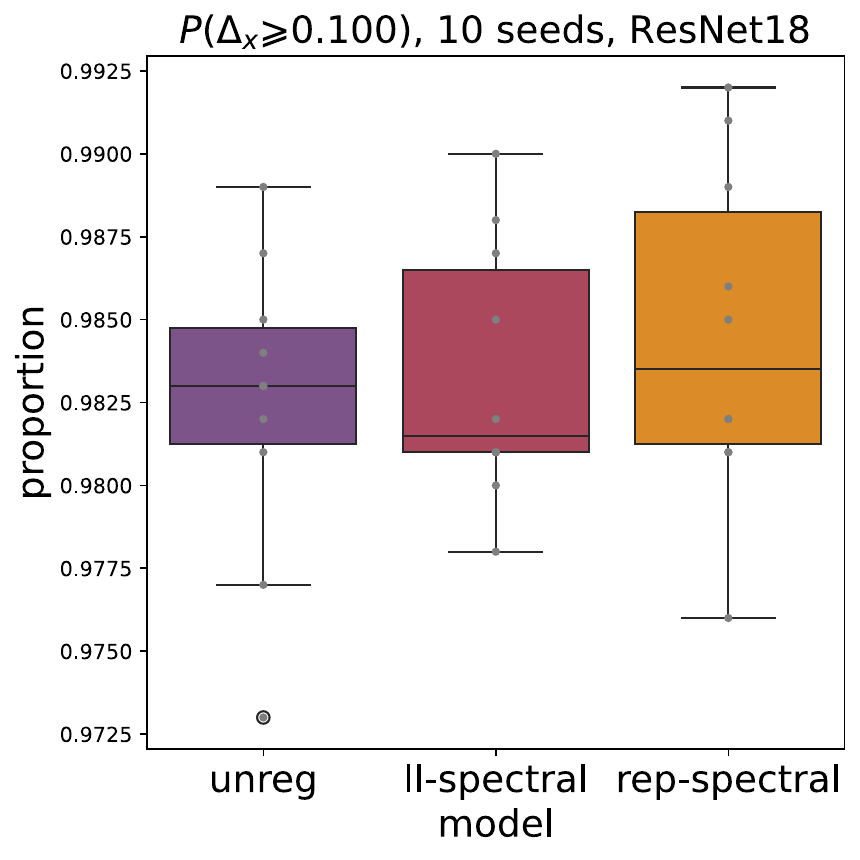}
    \end{subfigure} \hspace{0.5in}
    \begin{subfigure}{1.5in}
        \centering
        \includegraphics[width=\linewidth]{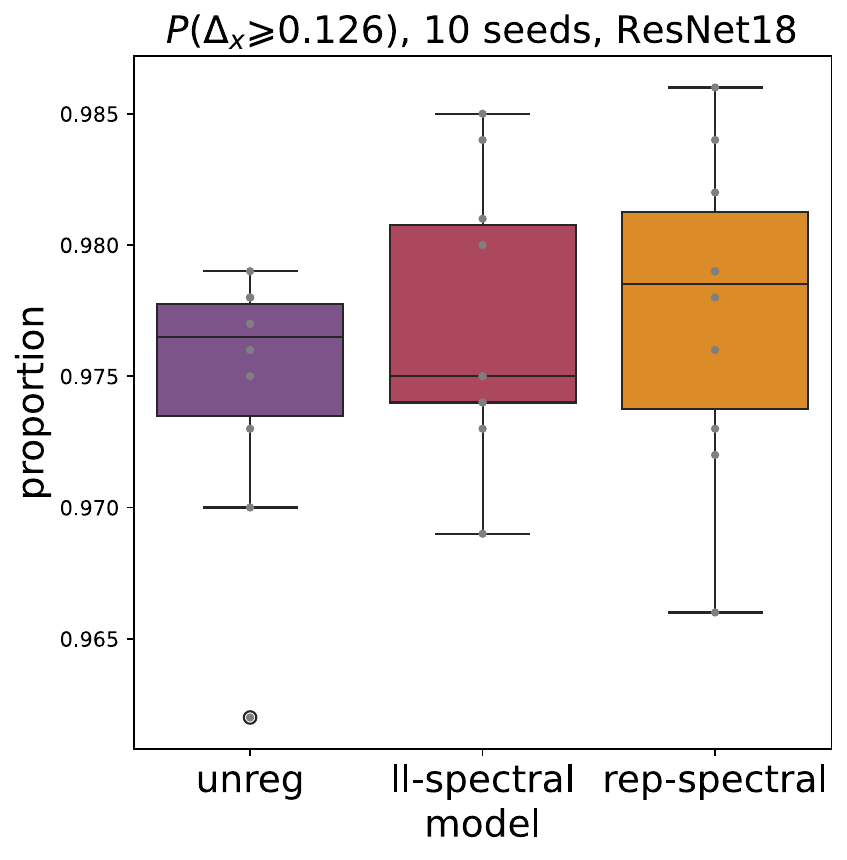}
    \end{subfigure} \hspace{0.5in}
    \begin{subfigure}{1.5in}
        \centering
        \includegraphics[width=\linewidth]{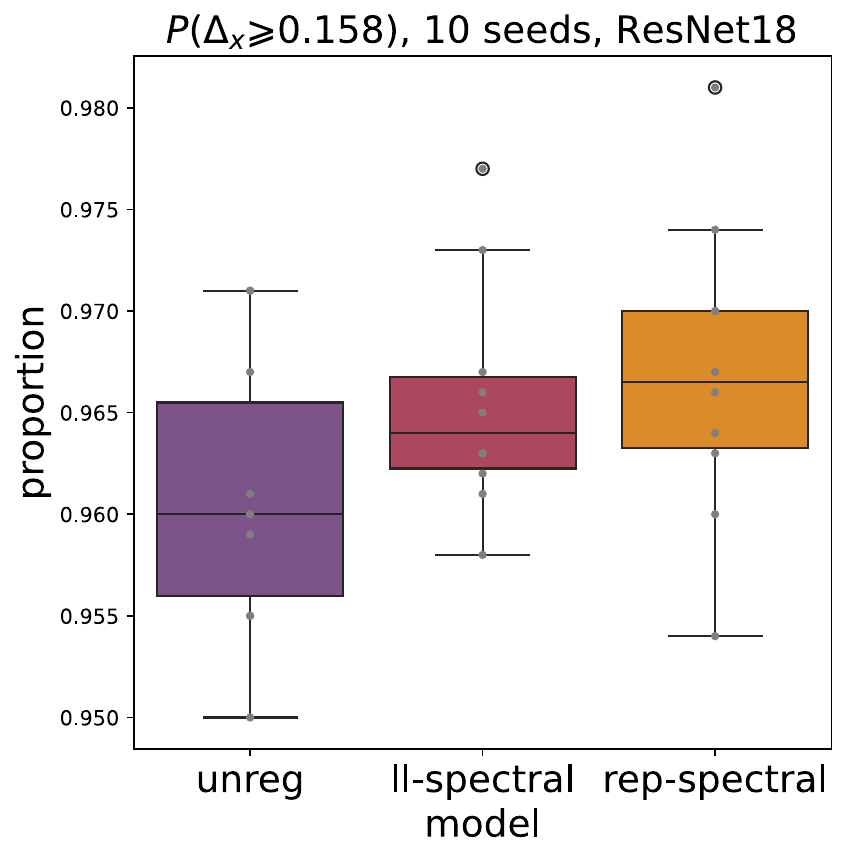}
    \end{subfigure} \\
    \begin{subfigure}{1.5in}
        \centering
        \includegraphics[width=\linewidth]{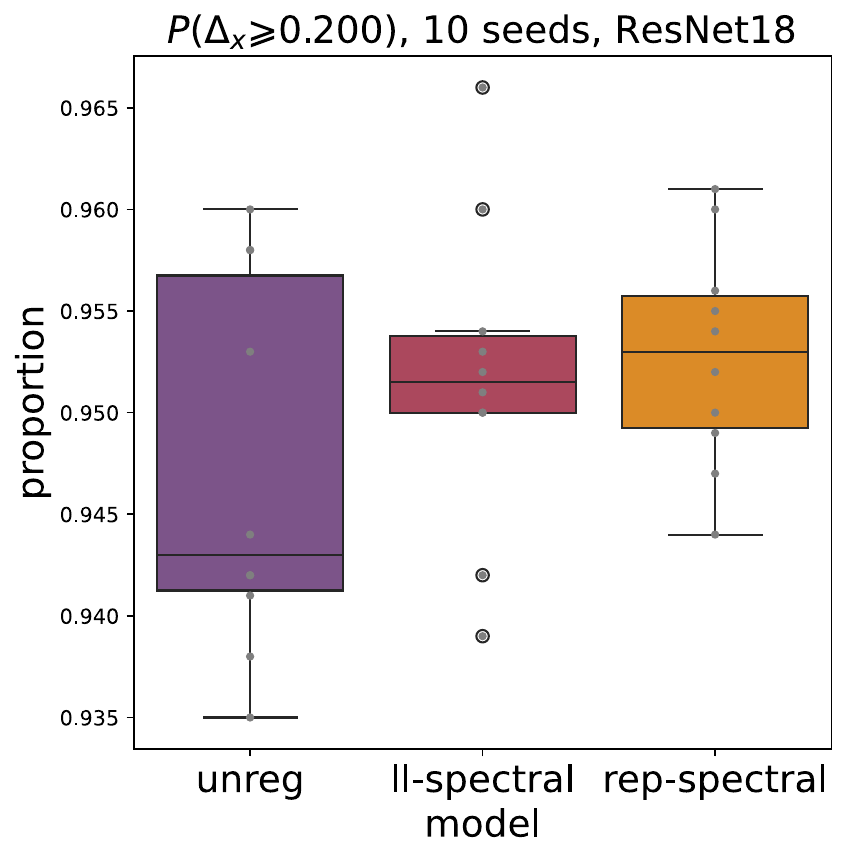}
    \end{subfigure} \hspace{0.5in}
    \begin{subfigure}{1.5in}
        \centering
        \includegraphics[width=\linewidth]{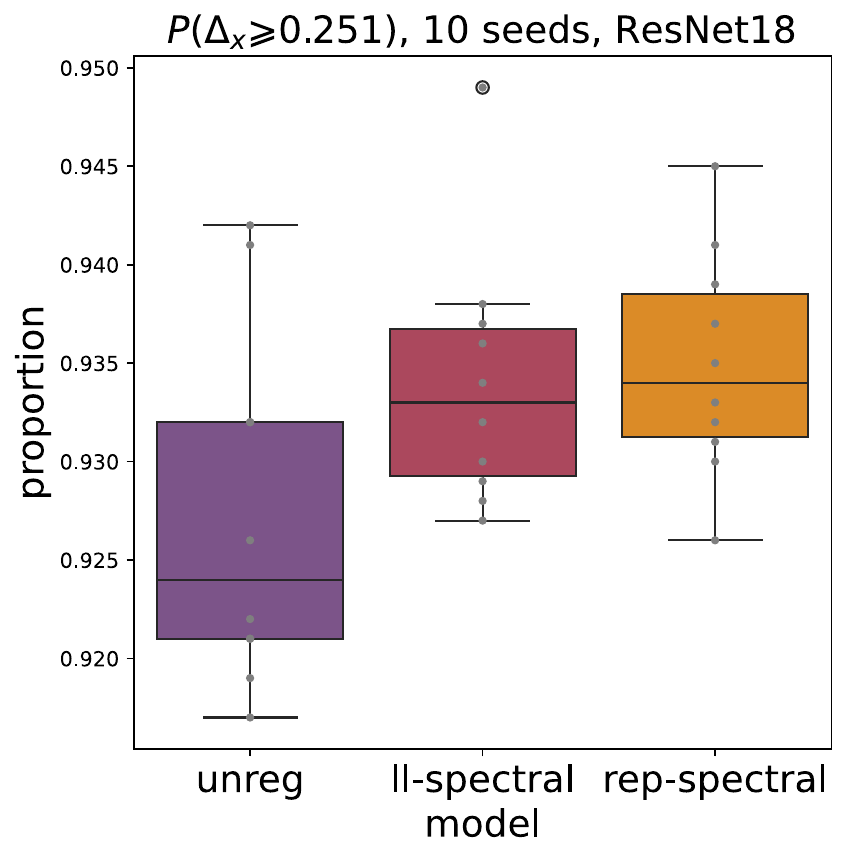}
    \end{subfigure} \hspace{0.5in}
    \begin{subfigure}{1.5in}
        \centering
        \includegraphics[width=\linewidth]{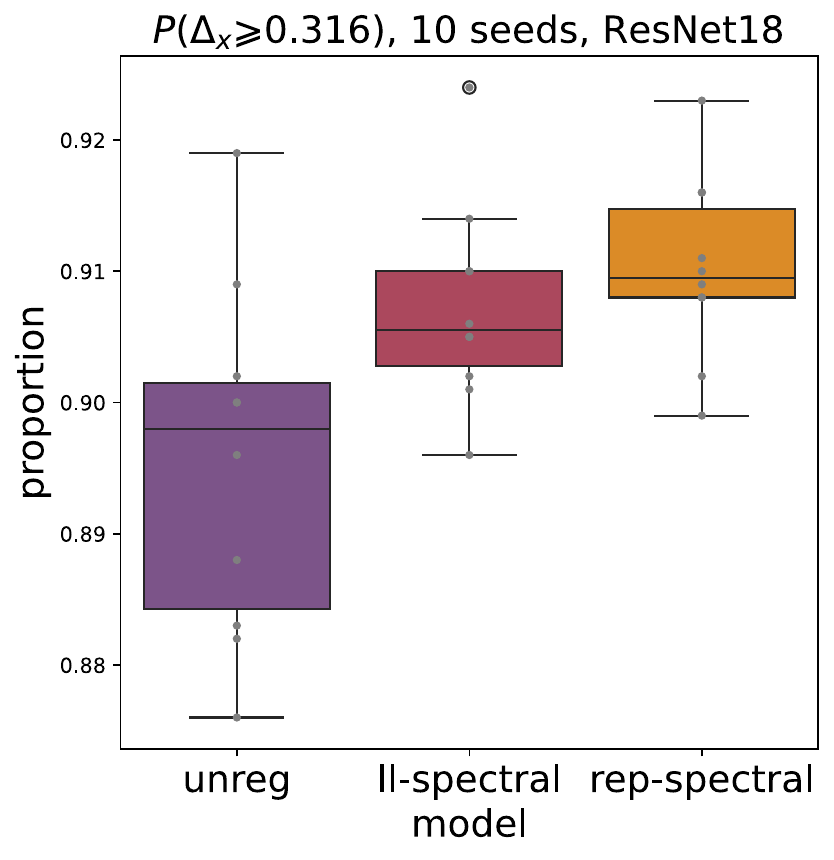}
    \end{subfigure} \\
    \begin{subfigure}{1.5in}
        \centering
        \includegraphics[width=\linewidth]{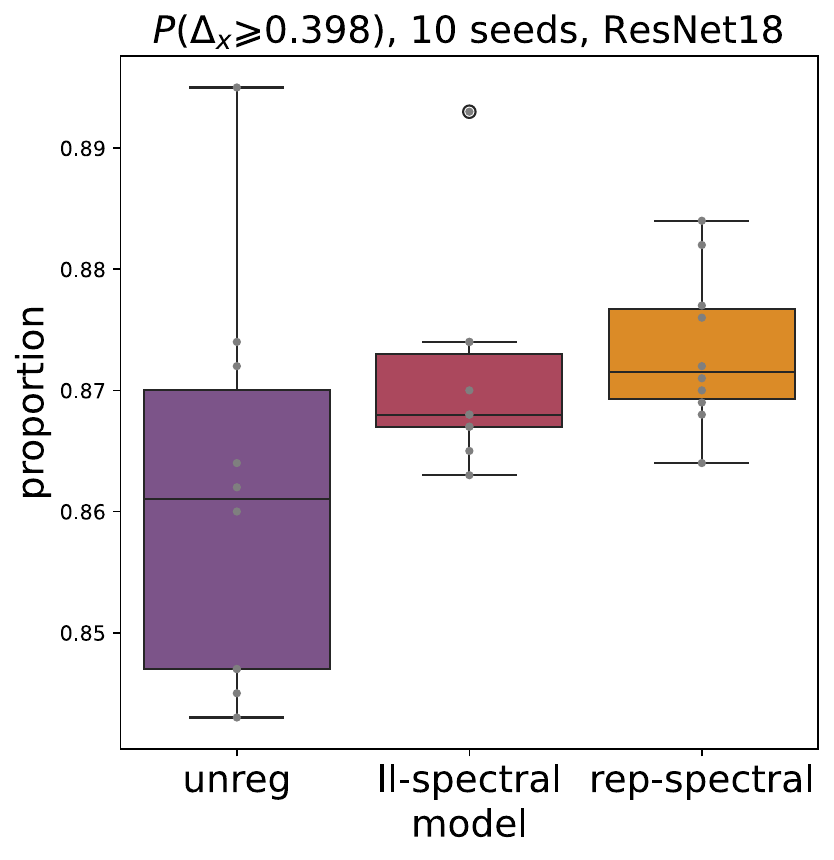}
    \end{subfigure} \hspace{0.5in}
    \begin{subfigure}{1.5in}
        \centering
        \includegraphics[width=\linewidth]{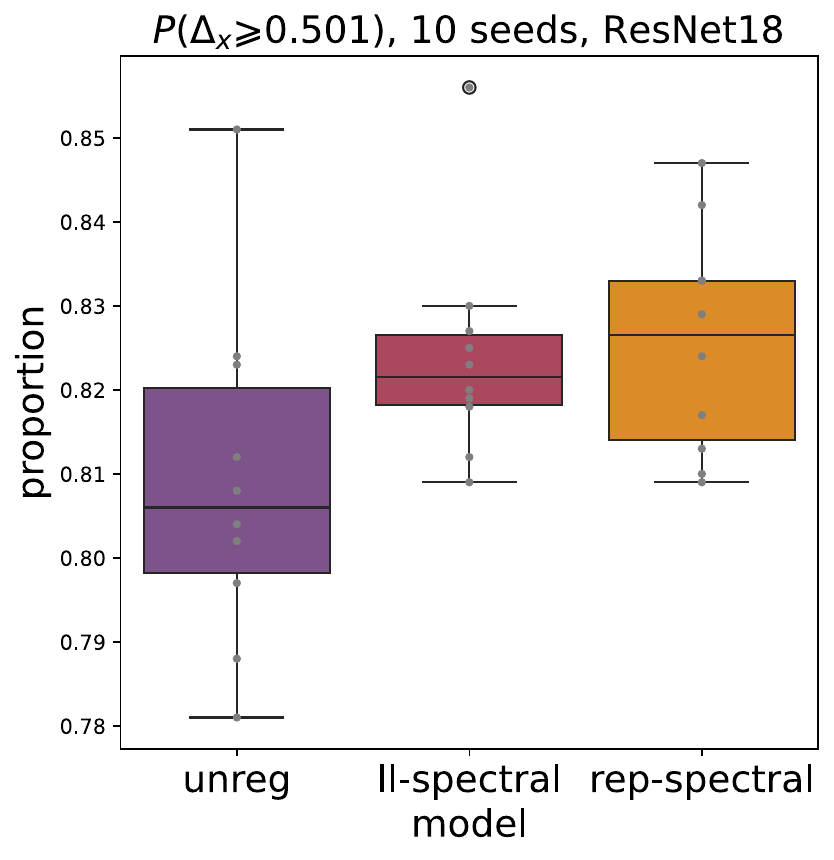}
    \end{subfigure} \hspace{0.5in}
    \begin{subfigure}{1.5in}
        \centering
        \includegraphics[width=\linewidth]{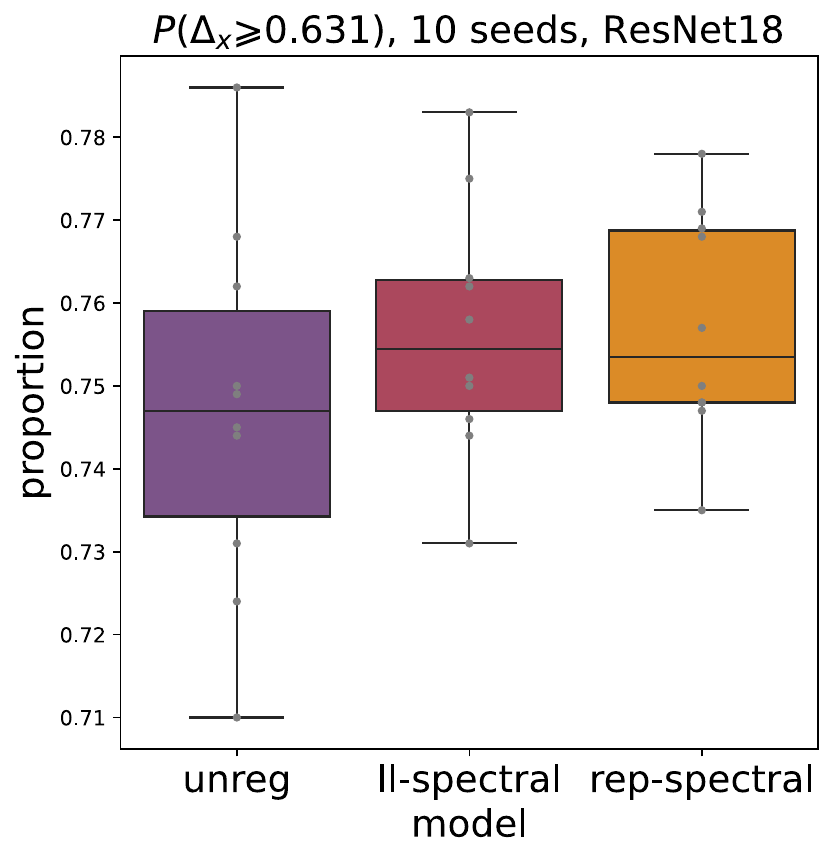}
    \end{subfigure} \\
    \begin{subfigure}{1.5in}
        \centering
        \includegraphics[width=\linewidth]{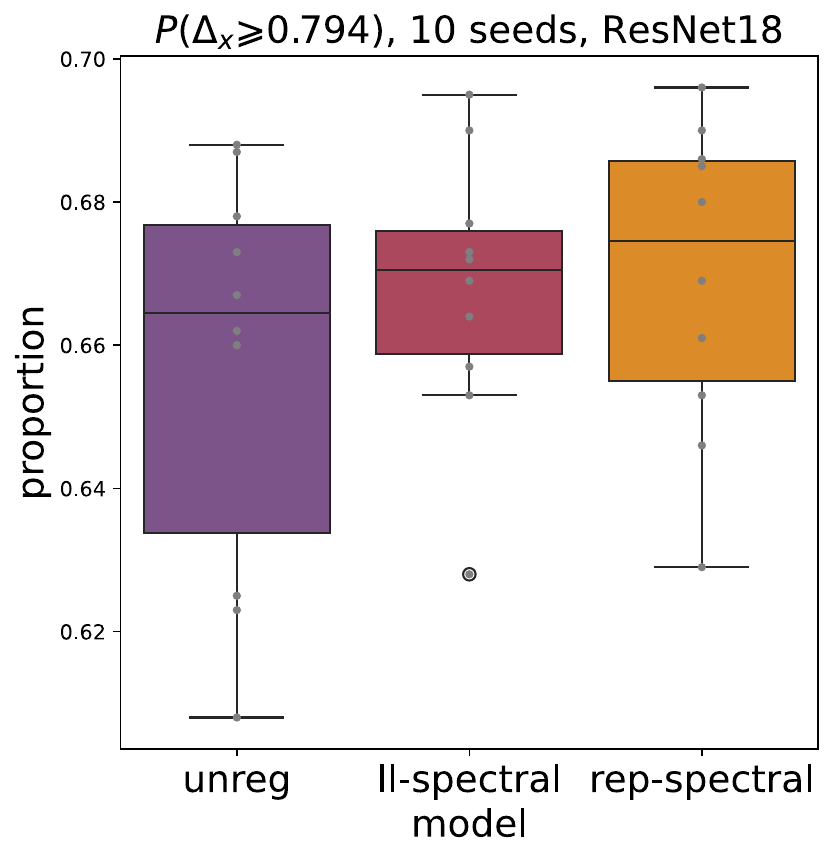}
    \end{subfigure} \hspace{0.5in}
    \begin{subfigure}{1.5in}
        \centering
        \includegraphics[width=\linewidth]{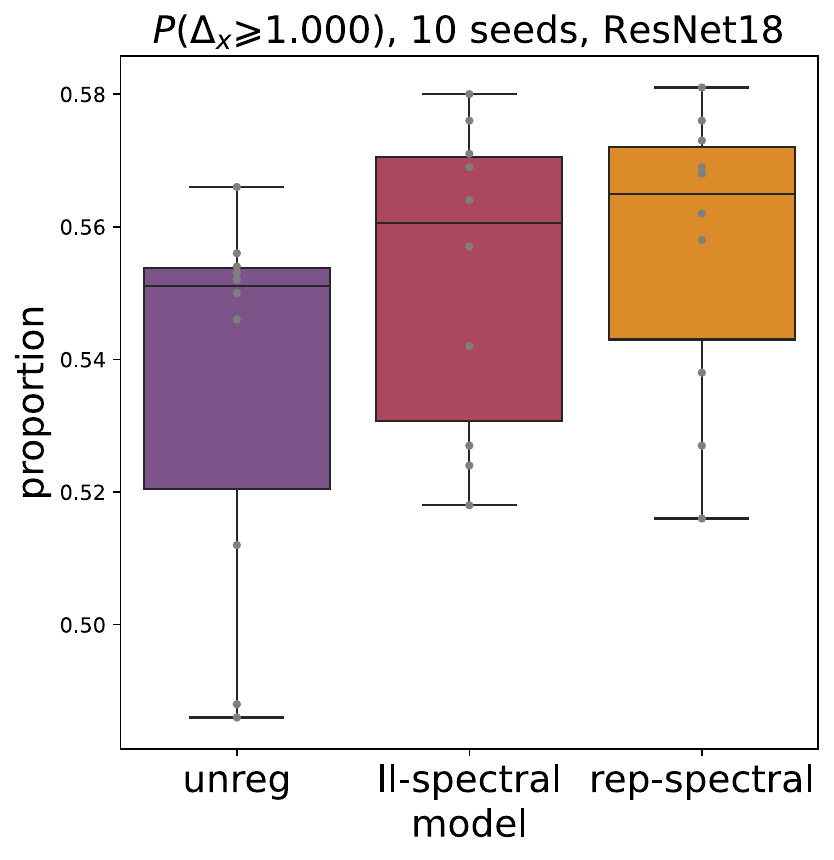}
    \end{subfigure} \hspace{0.5in}
    \begin{subfigure}{1.5in}
        \centering
        \includegraphics[width=\linewidth]{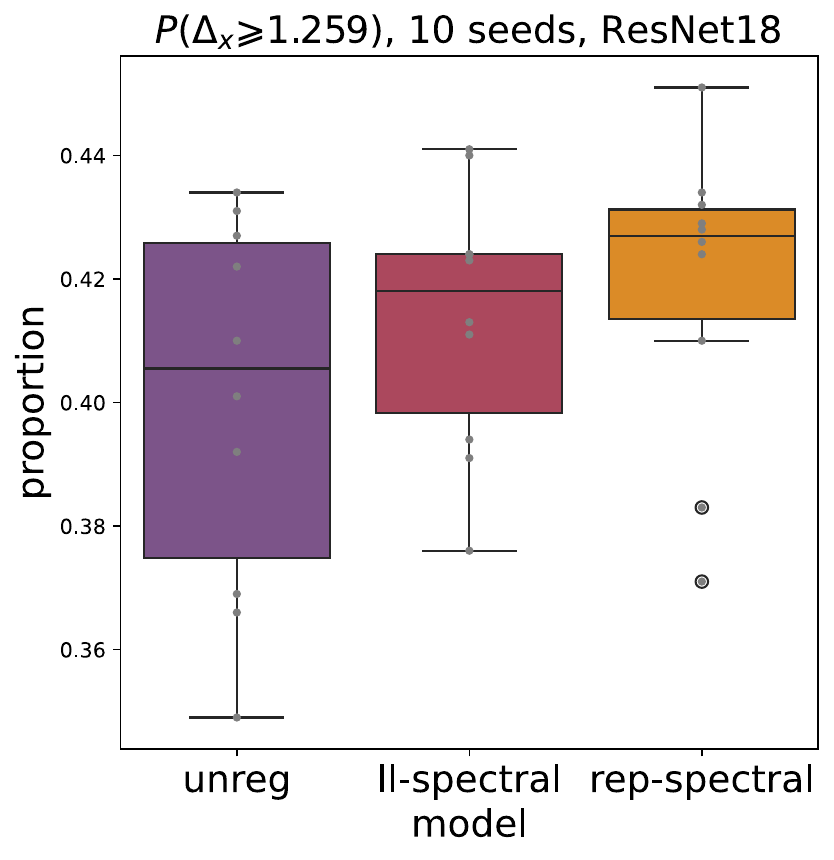}
    \end{subfigure} \\
    \begin{subfigure}{1.5in}
        \centering
        \includegraphics[width=\linewidth]{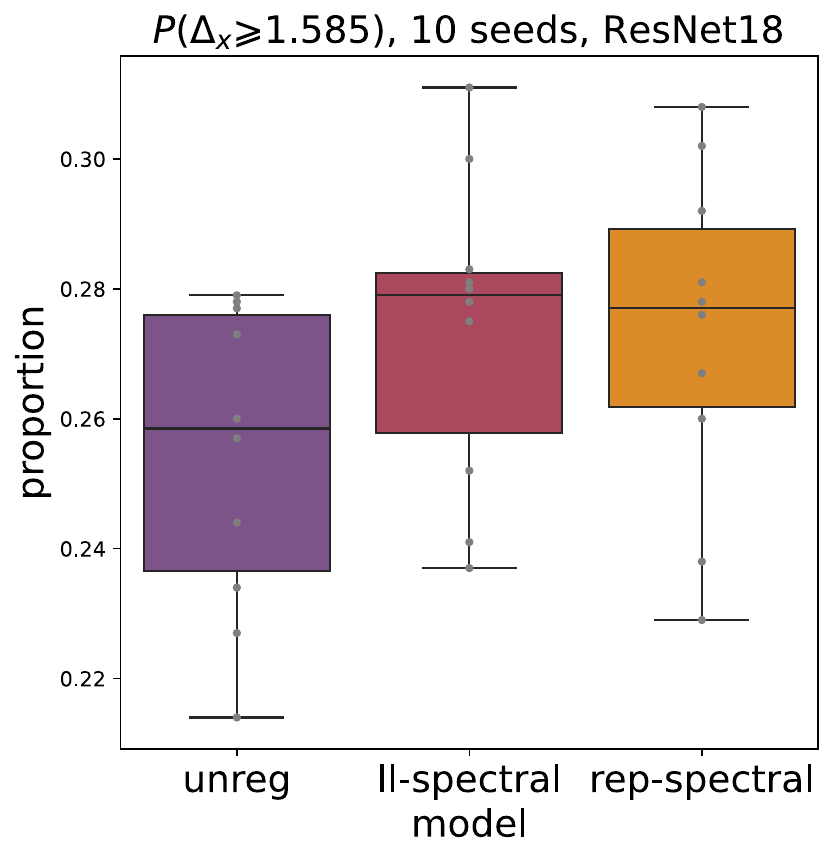}
    \end{subfigure} \hspace{0.5in}
    \begin{subfigure}{1.5in}
        \centering
        \includegraphics[width=\linewidth]{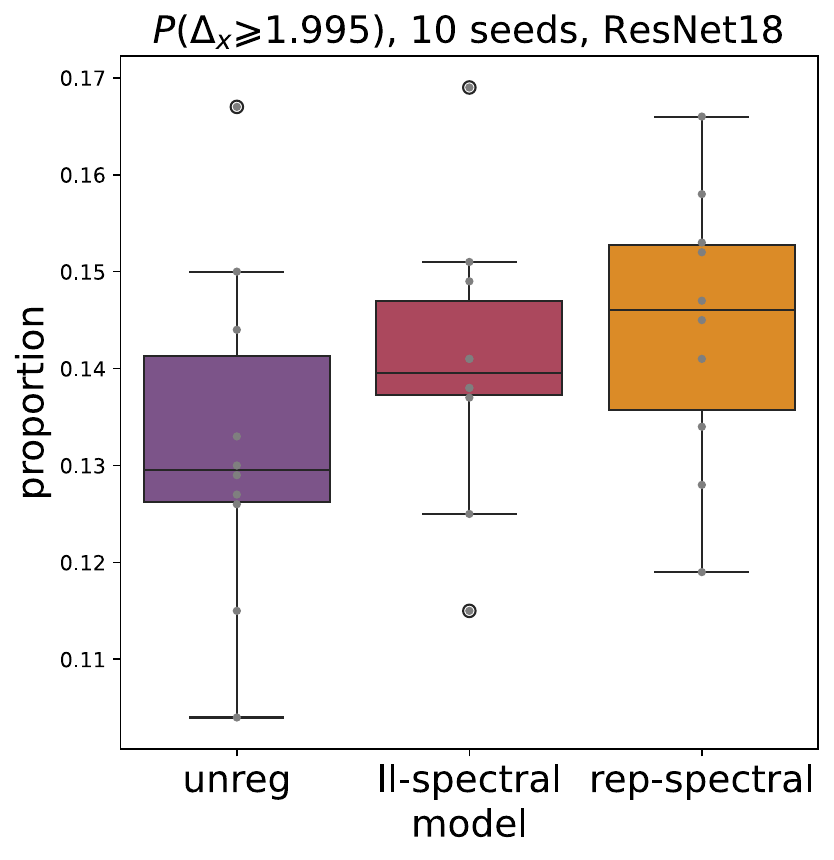}
    \end{subfigure} \hspace{0.5in}
    \begin{subfigure}{1.5in}
        \centering
        \includegraphics[width=\linewidth]{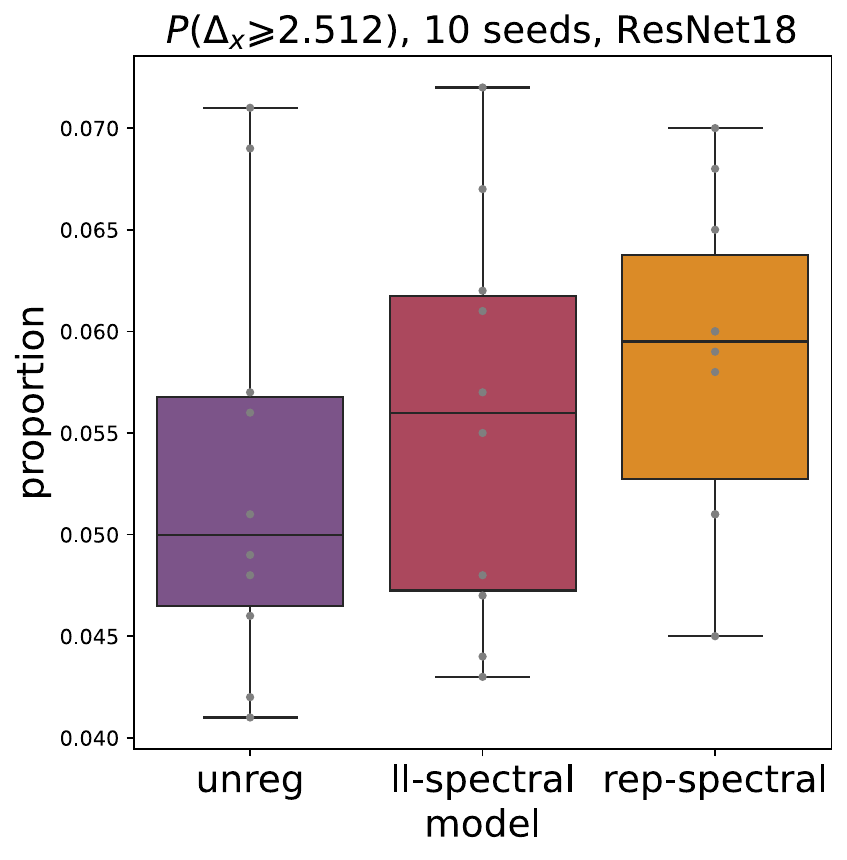}
    \end{subfigure} \\
    \caption{Proportion of 1000 randomly selected samples with adversarial distance larger than a threshold across 10 different seeds in training ResNet18s on CIFAR10, with the individual proportions colored in grey. 15 thresholds are displayed in increasing order. Higher proportion means better adversarial robustness. Across all thresholds, \texttt{rep-spectral} achieves the highest proportion.}
    \label{fig:resnet18_dist_prop}
\end{figure}

\subsection{Self-Supervised Learning}\label{sup:ssl}
In BarlowTwins training, using ResNet18, we train on CIFAR10 using SGD for 1000 epochs with learning rate 0.01, momentum 0.9, weight decay 1e-4, $\gamma = 0.01$, and a regularization burnin period of 900 epochs (90\% unregularized training + 10\% regularized training). Likewise, to alleviate training cost, we amortize power iteration to perform once every 24 parameter updates.

\subsection{Transfer Learning}\label{sup:tl}

In unregularized pretraining, we train ResNet50 for 120 epochs using SGD with learning rate 0.02, momentum 0.9, weight decay 1e-4, linear scheduling 30 epochs with decay 0.1. We distribute batchsize of 512 images across 4 GPUs for training. For regularized pretraining, starting at epoch 80 (67\% unregularized training + 33\% regularized training), we turn on \texttt{rep-spectral} regularization with $\gamma = 0.001$, with power update of top eigenvectors every 160 parameter updates to alleviate computation costs. 

We test these two models' performance at finetune time, in which we train on CIFAR10 scale to 224 by 224 images with the same dimensionality with ImagNet for 50 epochs using SGD with 0.01 learning rate on linear head and 0.002 learning rate for the backbone, each with a ConsineAnneling scheduling with max parameter 200 epochs for both backbone and linear head. We visualize the test accuracy for each model of each random seed in \Cref{fig:tl_supp_acc}. On average, finetuned model starting with regularized weights have 0.2\% drop in test accuracy than the ones staring with unregularized weights.

\begin{figure}[t]
    \centering
    \includegraphics[width=\textwidth]{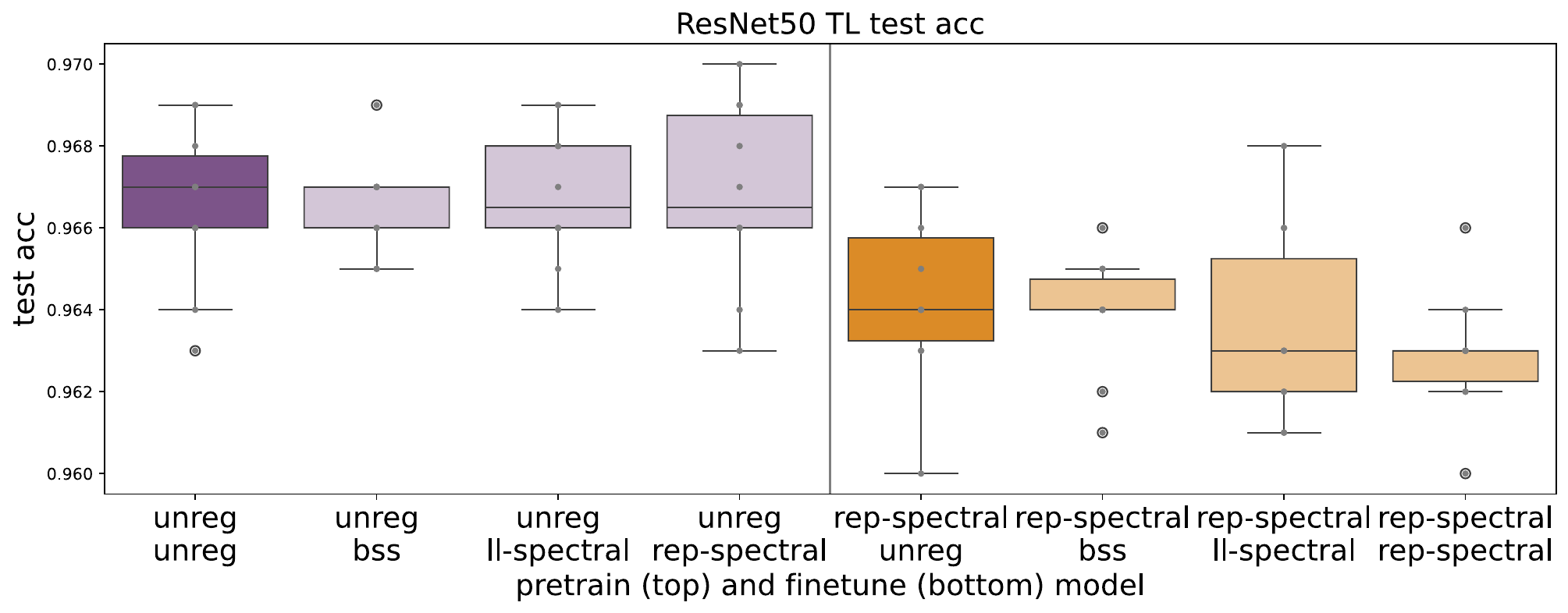}
    \caption{Test accuracy in transfer learning across multiple combinations of training schemes. The left half are finetuning from unregularized model, and the right half are finetuning from \texttt{rep-spectral} regularized model. All models reach an accuracy level of 96\%, but have different adversarial robustness level shown in \Cref{fig:tl_supp}.}
    \label{fig:tl_supp_acc}
\end{figure}

We also tested performance on more commonly used transfer learning target datasets: Stanford Dog \citep{KhoslaYaoJayadevaprakashFeiFei_FGVC2011} contains images of 120 different kinds of dogs; Oxford Flowers \citep{Nilsback08} contains images of 102 different types of flowers; MIT Indoor \citep{quattoni2009recognizing} contains indoor scenes of 67 different categories. Three dataset have the same input dimensionality with ImageNet1k. Similar as in finetuning on CIFAR10, we the model pretrained on ImageNet either with or without our regularizations, and in the finetunning stage we perform normal CrossEntropy optimization. As different target dataset have different inherent complexity, we finetune on the three models using different set of hyperparameters. To finetune on Stanford dog, we train for 50 epochs using SGD with a learning rate 0.005 for last layer and 0.001 for the feature layers; to finetune on Oxford Flowers we train for 1000 epochs using SGD with a learning rate 0.005 for last layer and 0.001 for the feature layer; and to finetune on MIT Indoor, we train for 100 epochs using SGD with a learning rate 0.01 and 0.002 for the feature layers. All dataset are trained with a btachsize of 64. With either pretrained weights with or without regularization, we repeat each training for 5 times and report the end test accuracy and adversarial distances in \Cref{fig:tl-others}. Although finetuned models starting from regularized weights have 2\% drop in test accuracy compared to the finetuned models starting with unregularized weights, we have roughly 10\% increase in mean adversarial distances across different dataset, suggesting that our feature map trained is not uniquely robust to finetuning on one particular dataset. 

\begin{figure}
    \centering
    \begin{subfigure}{2in}
        \centering
        \includegraphics[width=\linewidth]{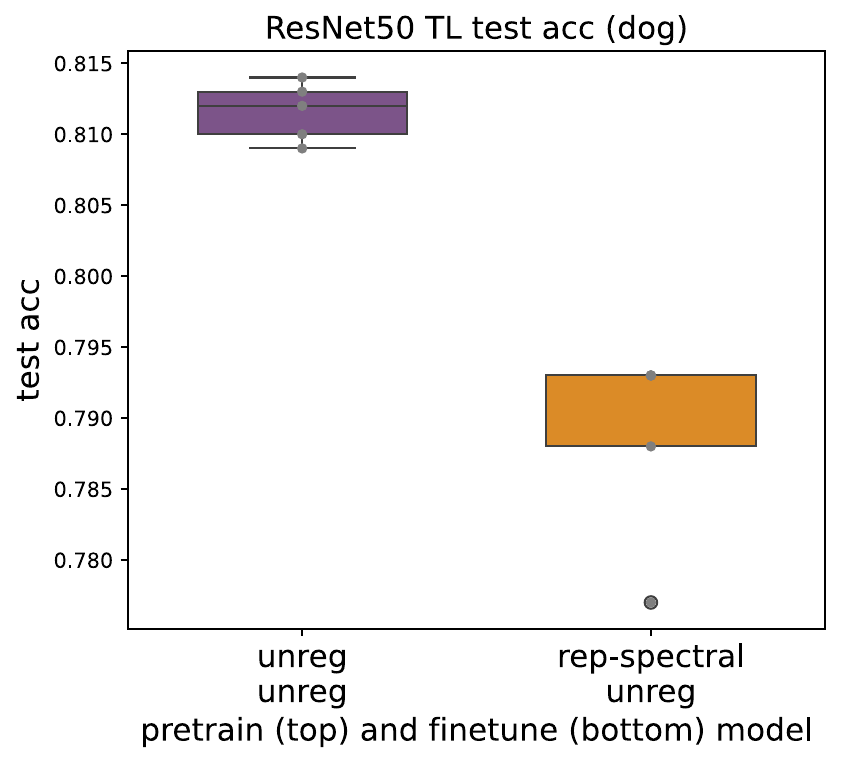}
    \end{subfigure} \hspace{0.5in}
    \begin{subfigure}{2in}
        \centering
        \includegraphics[width=\linewidth]{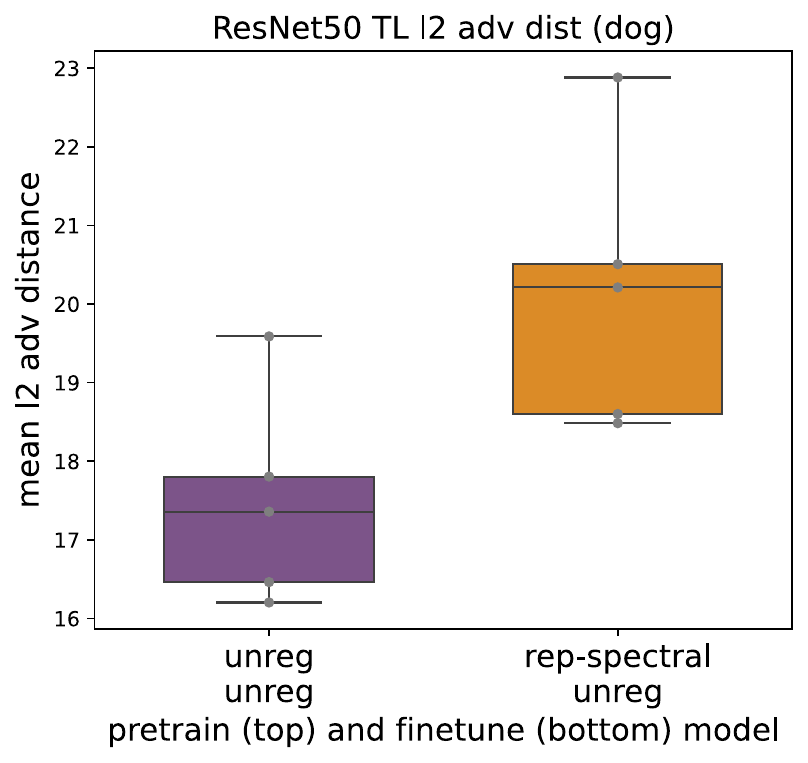}
    \end{subfigure} \\
    \begin{subfigure}{2in}
        \centering
        \includegraphics[width=\linewidth]{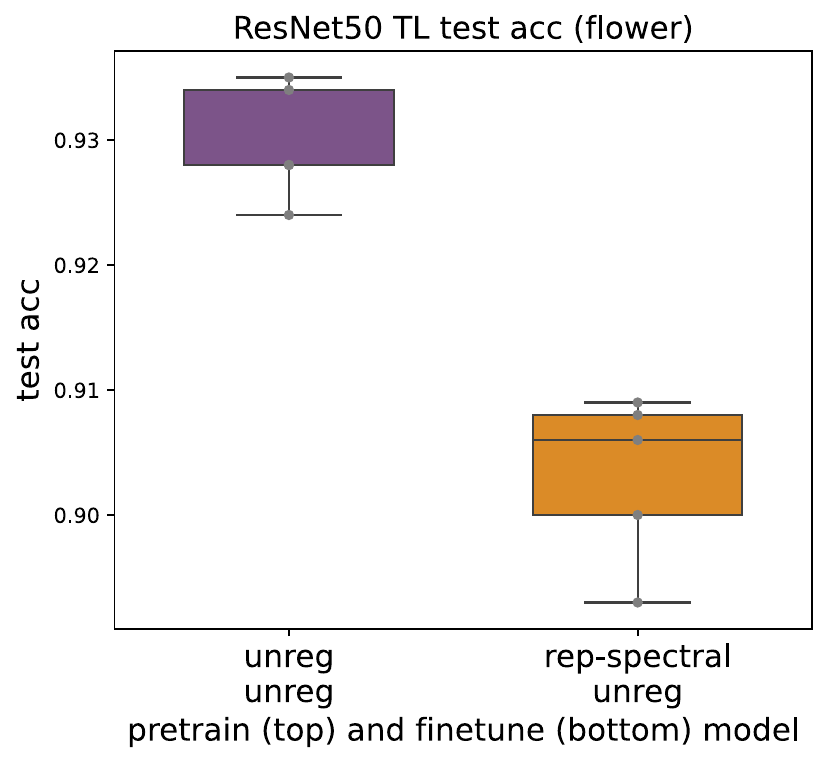}
    \end{subfigure} \hspace{0.5in}
    \begin{subfigure}{2in}
        \centering
        \includegraphics[width=\linewidth]{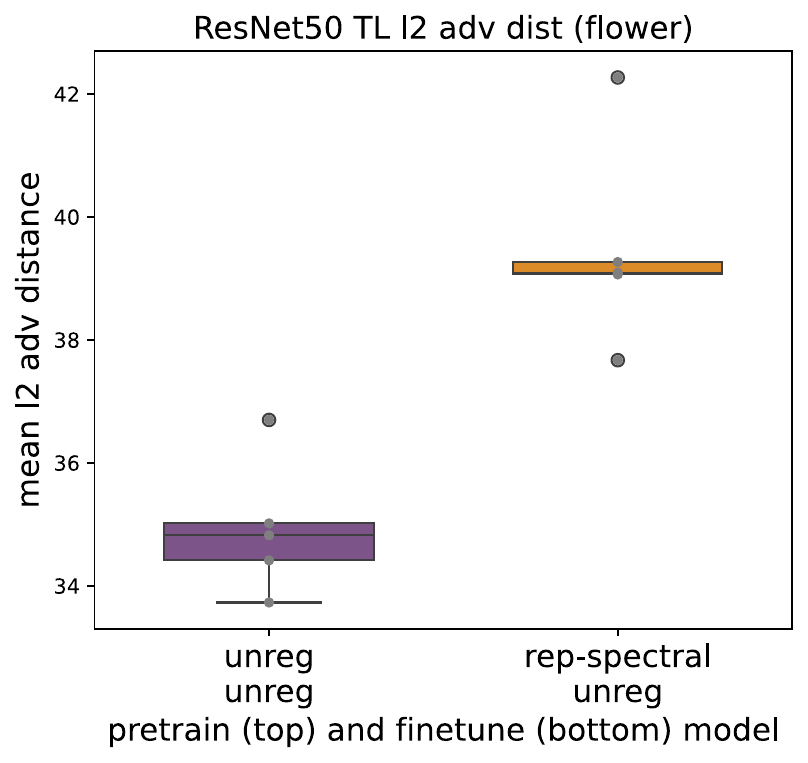}
    \end{subfigure} \\
    \begin{subfigure}{2in}
        \centering
        \includegraphics[width=\linewidth]{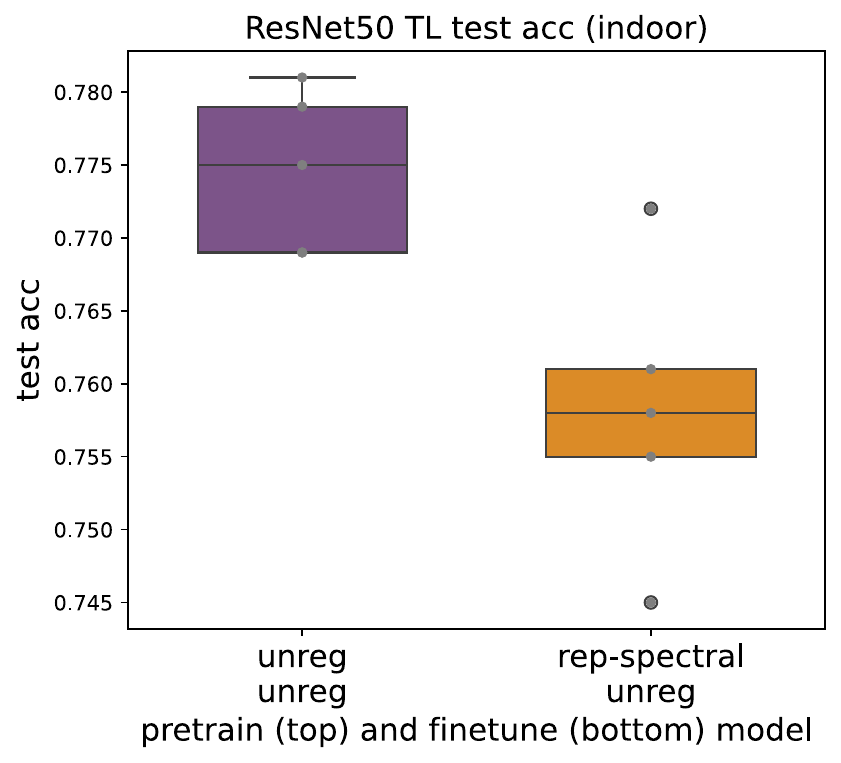}
    \end{subfigure} \hspace{0.5in}
    \begin{subfigure}{2in}
        \centering
        \includegraphics[width=\linewidth]{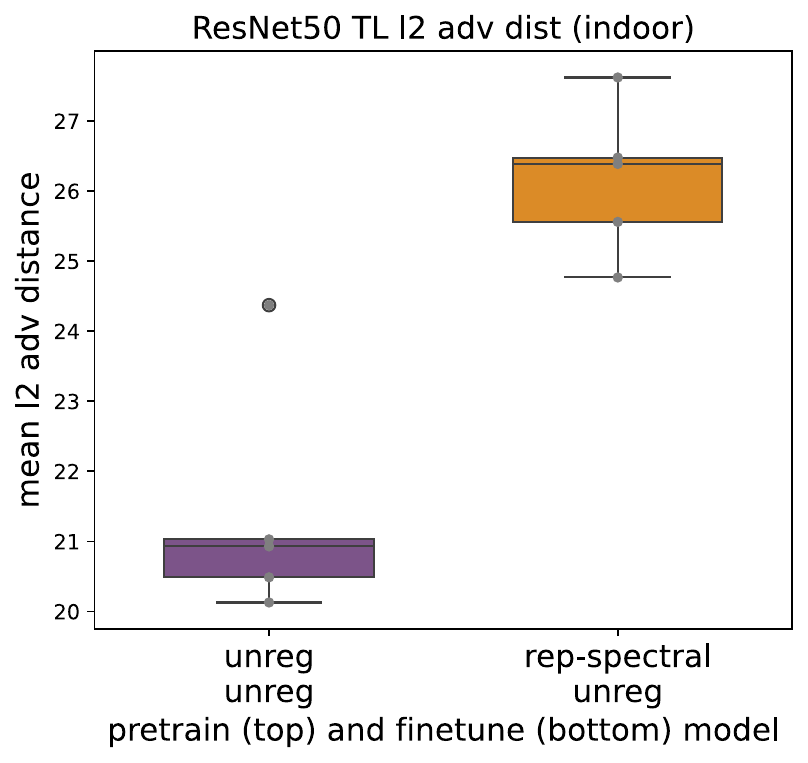}
    \end{subfigure} \\
    \caption{Test accuracy and mean adversarial distances from pretraining on ImageNet and finetuning on Stanford Dog (top row), Oxford Flowers (middle row), MIT Indoor (bottom row). The grey dots indicate the value from each of the 5 random seeds. Sacrificing at most 2\% of test accuracy, we obtain at least 10\% gain in the adversarial distances on average.}
    \label{fig:tl-others}
\end{figure}

\end{document}